\documentclass{article} %
\usepackage{iclr2024_conference,times}

\usepackage{amsmath,amsfonts,bm}

\def\eqref#1{equation~\ref{#1}}

\def\1{\bm{1}}

\DeclareMathAlphabet{\mathsfit}{\encodingdefault}{\sfdefault}{m}{sl}
\SetMathAlphabet{\mathsfit}{bold}{\encodingdefault}{\sfdefault}{bx}{n}

\usepackage{amsthm}
\usepackage{hyperref}
\usepackage{url}
\usepackage{cleveref}
\usepackage{url}            %
\usepackage{booktabs}       %
\usepackage{amsfonts}       %
\usepackage{nicefrac}       %
\usepackage{microtype}      %
\usepackage{xcolor}         %
\usepackage{multirow}
\usepackage{graphicx}
\usepackage{tabularx}
\usepackage{amsmath,amssymb}
\usepackage{bm}
\usepackage{multirow}
\usepackage{comment}
\usepackage{booktabs}
\usepackage{pifont}
\usepackage{float}
\usepackage{appendix}
\usepackage{caption}
\usepackage[normalem]{ulem} %
\usepackage{subcaption}
\usepackage{colortbl}

\crefname{ineq}{inequ.}{inequ.}
\crefname{equation}{Eq.}{Eqs.}
\creflabelformat{ineq}{#2{\upshape(#1)}#3} 
\crefname{theorem}{Thm.}{Thm.}
\crefname{proposition}{Prop.}{Prop.}
\crefname{claim}{Claim}{Claims}
\crefname{lemma}{Lemma}{Lemmas}
\crefname{assumption}{Asm.}{Asm.}
\crefname{figure}{Fig.}{Fig.}
\crefname{appendix}{Appendix}{Appendices}
\crefname{table}{Table}{Tables}
\crefname{section}{Section}{Sections}

\newtheorem{proposition}{Proposition}

\newcommand{\rebuttal}[1]{#1}

\title{Multilinear Operator Networks}

\author{Yixin Cheng$^{1}$, Grigorios G. Chrysos$^{2}$, Markos Georgopoulos$^{}$, \textbf{Volkan Cevher}$^{1}$ \\
$^1$LIONS - École Polytechnique Fédérale de Lausanne \hspace{1em} $^2$University of Wisconsin-Madison
}

\providecommand{\realnum}{\ensuremath{\mathbb{R}}}
\providecommand{\naturalnum}{\ensuremath{\mathbb{N}}}

\newcommand{\modelnamePM}{MONet}
\newcommand{\modelnamePMS}{\modelnamePM-T}
\newcommand{\modelnamePMB}{\modelnamePM-S}

\newcommand{\imagenet}{ImageNet1K}
\newcommand{\abdataset}{ImageNet100}
\newcommand{\layer}{Mu-Layer}
\newcommand{\block}{Poly-Block}
\definecolor{small}{rgb}{0.99, 0.88, 0.88}
\definecolor{base}{rgb}{0.88, 0.99, 0.88}
\definecolor{big}{rgb}{0.88,0.88,0.99}

\newcommand{\modelnamePI}{$\Pi$-Nets}

\newcommand{\modelnamedense}{$\mathcal{D}$-PolyNets}  %
\newcommand{\newmodelnamePI}{$\mathcal{R}$-PolyNets}
\newcommand{\resnet}{ResNet}

\providecommand{\shortpolynamesingle}{PN}
\providecommand{\shortpolyname}{\shortpolynamesingle s}

\newcommand{\myth}{\ensuremath{^{\text{th}}}}

\makeatletter
\newcommand{\period}{\char`.\@ifnextchar'{}{\@ifnextchar.{}{ }}}
\newcommand{\periodcomma}{\char`,\@ifnextchar'{}{\@ifnextchar.{}{ }}}
\makeatother

\begin{document}

\maketitle

\begin{abstract}
Despite the remarkable capabilities of deep neural networks in image recognition, the dependence on activation functions remains a largely unexplored area and has yet to be eliminated.
On the other hand, Polynomial Networks is a class of models that does not require activation functions, but have yet to perform on par with modern architectures. In this work, we aim close this gap and propose \modelnamePM, which relies \emph{solely} on multilinear operators. The core layer of \modelnamePM, called \layer, captures multiplicative interactions of the elements of the input token. \modelnamePM{} captures high-degree interactions of the input elements and we demonstrate the efficacy of our approach on a series of image recognition and scientific computing benchmarks. The proposed model outperforms prior polynomial networks and performs on par with modern architectures. We believe that \modelnamePM{} can inspire further research on models that use entirely multilinear operations.

\end{abstract}

\section{Introduction}
\label{sec:poly_mixer_introduction}

Image recognition has long served as a crucial benchmark for evaluating architecture designs, \rebuttal{including the seminal \resnet~\citep{he2016resnet} and MLP-Mixer \citep{tolstikhin2021mlp}. As architectures are applied to new applications, there are additional requirements for the architecture design. For instance, encryption is a key requirements in safety-critical applications~\citep{caruana2015intelligible}. 
Concretely, the Leveled Fully Homomorphic Encryption (LFHE)~\citep{brakerski2014leveled}, can provide a high level of security for sensitive information. The core limitation of FHE (and especially LFHE) is that they support only addition and multiplication as operations. That means that traditional neural networks cannot fit into such privacy constraints owing to their dependence on elementwise activation functions, making developments in MLP-Mixer and similar models invalid for many real-world applications.} Therefore, new designs that can satisfy those constraints and still achieve high accuracy on image recognition are required. 

\rebuttal{A core advantage of Polynomial Nets (\shortpolyname), that express the output as high-degree interactions between the input elements~\citep{ivakhnenko1971polynomial, shin1991pi-sigma, chrysos2020pinet}, is that they can satisfy constraints, such as encryption or interpretability~\citep{dubey2022scalable}.} However, a major drawback of \shortpolyname{} so far is that they fall short of the performance of modern architectures on standard machine learning benchmarks, such as image recognition. This is precisely the gap we aim to close in this work. 

We introduce a class of \shortpolyname, dubbed \textcolor{blue}{M}ultilinear \textcolor{blue}{O}perator \textcolor{blue}{Net}work (\modelnamePM), which is based solely on multilinear operations\footnote{\rebuttal{The terminology on multilinear operations arises from the multilinear algebra. Concretely, we follow the terminology of the seminal paper of \cite{kolda2009tensor}.}}. The core layer captures multiplicative interactions within the token elements\footnote{Consistent with the recent literature of MLP-based models and transformers~\citep{dosovitskiy2020ViT}, we consider sequences of tokens as inputs. In the case of images, a token refers to a (vectorized) patch of the input image.}. The multiplicative interaction is captured using two parallel branches, each of which assumes a different rank to enable different information to flow. By composing sequentially such layers, the model captures high-degree interactions between the input elements and can predict the target signal, e.g., class in the case of image recognition. 

Concretely, our contributions can be summarized as: 
\begin{itemize}
    \item We propose \layer, a new module which uses purely multilinear operations to capture multiplicative interactions. We showcase how this module can serve as a plug-in replacement to standard MLP. 
    
    \item We introduce \modelnamePM{}, which captures high-degree interactions between the input elements. To our knowledge, this is the first network that obtains a strong performance on challenging benchmarks. 
    
    \item We conduct a thorough evaluation of the proposed model across standard image recognition benchmarks to show the efficiency and effectiveness of our method. \modelnamePM{} significantly improves the performance of the prior art on polynomial networks, while it is on par with other recent strong-performing architectures.
\end{itemize}

We intend to release the source code upon the acceptance of the paper to enable further improvement of models relying on linear projections.

\section{Related work}
\label{sec:poly_mixer_related}

We present a brief overview of the most closely related categories of MLP-based and polynomial-based architectures from the vast literature of deep learning architectures. For a detailed overview, the interested reader can find further information on dedicated surveys on image recognition~\citep{lu2007survey,plested2022deep,peng2022survey}. %

\textbf{MLP models}: %
The resurgence of MLP models in image recognition is an attempt to reduce the computational overhead of transformers~\citep{vaswani2017Transformer}. MLP-based models rely on linear layers, instead of convolutional layers or the self-attention block. The MLP-Mixer~\citep{tolstikhin2021mlp} is among the first networks that demonstrate a high accuracy on challenging benchmarks. The MLP-Mixer uses tokens (i.e., vectorized patches of the image) and captures both inter- and intra-token correlations. Follow-up works improve upon the simple idea of token-mixing MLP structure~\citep{touvron2022resmlp, liu2021pay}.  Concretely, ViP \citep{hou2022ViP}, cycleMLP \citep{chen2021cyclemlp}, S2-MLPv2 \citep{yu2021s2MLPV2} design strategies to improve feature aggregation across spatial positions. Our work differs from MLP-based models, as it is inspired by the idea of capturing high-order interactions using polynomial expansions.

\textbf{Polynomial models}: \rebuttal{Polynomial expansions establish connections between input variables and learnable coefficients through addition and multiplication. Polynomial Nets (\shortpolyname) express the output variable (e.g., class) as a high-degree expansion of the input elements (e.g., pixels of an image) using learnable parameters. Even though extracting such polynomial features is not a new idea~\citep{shin1991pi-sigma, li2003spsnn}, it has been largely sidelined from architecture design.}\looseness-1

\rebuttal{The last few years \shortpolyname{} have demonstrated promising performance in standard benchmarks in various vision tasks including image recognition~\citep{chrysos2022augmenting}. In particular, \shortpolyname{} augmented with activation functions, which are referred to as \emph{hybrid models} in this work, can achieve state-of-the-art performance~\citep{hu2018squeeze,li2019selective,yin2020disentangled, babiloni2021poly, yang2022polynomialnerf,georgopoulos2020multilinear,chrysos2021conditional,georgopoulos2021mitigating}. The work of \citet{chrysos2022augmenting} introduces a taxonomy for a single block of \shortpolyname{} based on the degree of interactions captured. This taxonomy allows the comparison of various approaches based on a specific degree of interactions. Building upon this work, researchers have explored ways to modify the degree or type of interactions captured to improve the performance. For example, \citet{babiloni2021poly} reduce computational cost of the popular non-local block~\citep{wang2018non} by capturing exactly the same third degree interactions, while \citet{chrysos2022augmenting} investigate modifications to the degree of expansion.}

\rebuttal{Arguably, the works most closely related to ours are \citet{chrysos2020pinet, chrysos2023regularization}. In \modelnamePI~\citep{chrysos2020pinet}, three different models are instantiated based on separate manifestations of three tensor decompositions. That is, the paper establishes a link between a concrete tensor decomposition and its assumptions to a concrete architecture that can be implemented in standard frameworks. The authors evaluate those different architectures and notice that they perform well, but do not surpass standard networks, such as ResNet. \citet{chrysos2023regularization} improve upon the \modelnamePI{} by introducing carefully designed (implicit and explicit) regularization techniques, which reduce the overfitting of high-degree polynomial expansions.} 

\rebuttal{Contrary to the aforementioned \shortpolyname, we are motivated to introduce a new architecture using \shortpolynamesingle{} that is comparable to modern networks. To achieve that, we are inspired by the modern setup of considering the input as a sequence of tokens. The token-based input is widely used across a range of domains and modalities that last few years. The token-based input also departs from the design of previous \shortpolyname{} that utilize convolutional layers instead of simple matrix multiplications. Arguably, our choice results in a weaker inductive bias as pointed out in the related works of MLP-based models. This can be particularly useful in domains outside of images, e.g., in the ODE experiments. }

\section{Method}
\label{sec:poly_mixer_method}
Let us now introduce \modelnamePM, which expresses the output as a polynomial expansion of the input using multilinear operations. Firstly, we introduce the core layer, called \layer, in \cref{subsec:Polymlp}. 
Then, in \cref{subsec:Network Architecture}, we design the whole architecture with a schematic visualized in \cref{fig:Poly Block}.

\textbf{Tokens}: Following the conventions of recent MLP-based models, we consider an image as a sequence of tokens. In practice, each token is a patch of the image. We denote the sequence of tokens as $\bm{X} \in \realnum^{d \times s}$, where $d \in \naturalnum$ is the length of a token and $s\in \naturalnum$ is the number of tokens. As such, our method below assumes an input sequence $\bm{X}$ is provided. MLP-based models capture linear interactions between the elements of a token. That operation would be denoted as $\bm{\Gamma}\bm{X}$, where $\bm{\Gamma} \in \realnum^{o\times d}$ is a learnable matrix.

\begin{figure}[tbp]
\captionsetup{font=small}
  \centering
 \includegraphics[width=0.8\columnwidth]{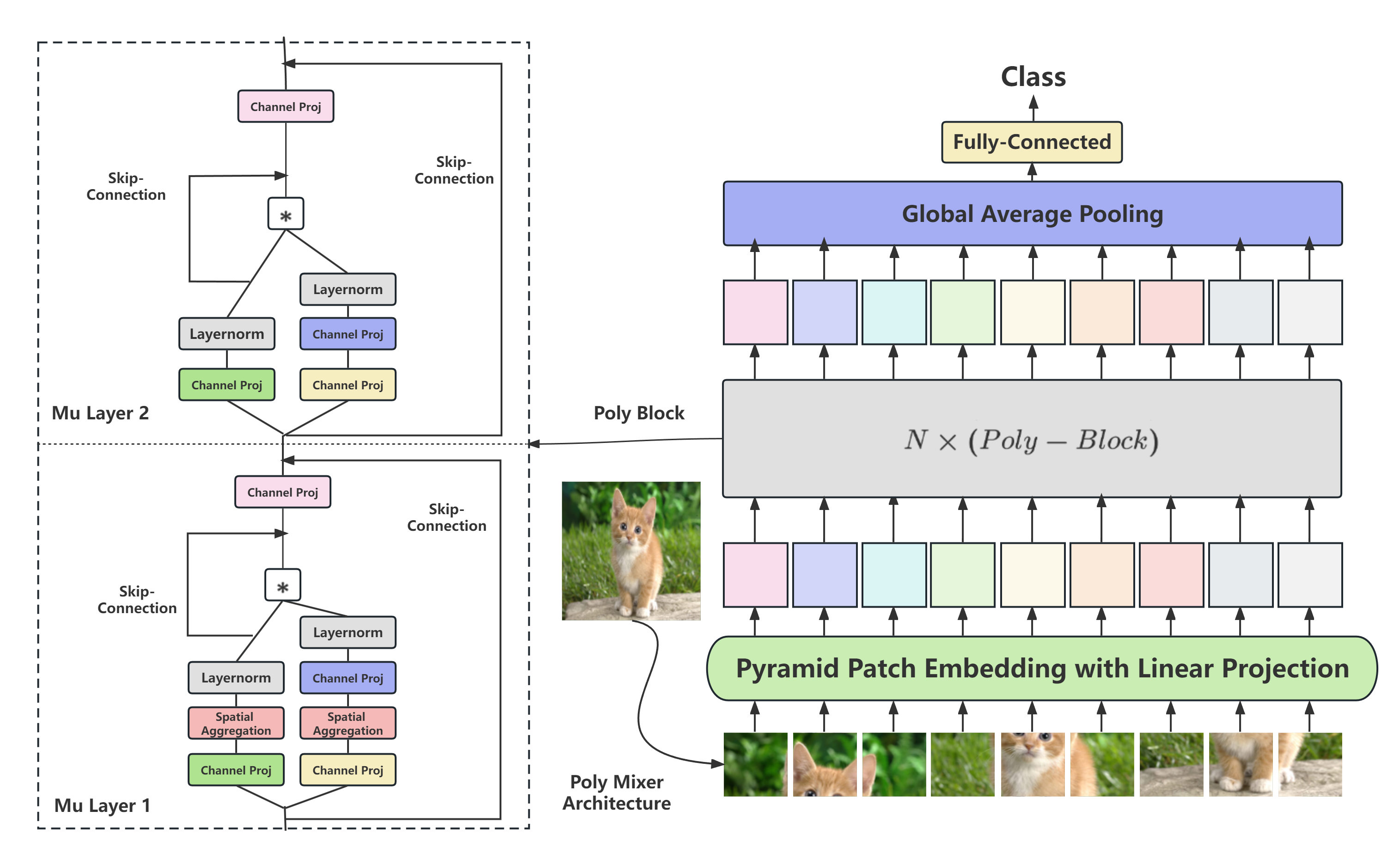}
  
  \caption{The architecture of the proposed \layer{} (on the left) and \modelnamePM{} (on the right). In the left figure, the grey box represents layer normalization.  The color solid line boxes represent channel projection in different dimensions, all projection operations are linear. The $\ast$  box denotes an elementwise (Hadamard) product. The red dash box represents the spatial aggregation module.}
  \label{fig:Poly Block}
  \vspace{-7mm}
\end{figure}

\subsection{\layer}
\label{subsec:Polymlp}
Our goal is to capture rich interactions among the input elements of a token. We can achieve that using multiplicative interactions. Specifically, we use two branches, and each captures a set of linear interactions inside the token. An elementwise product is then used to capture the multiplicative interactions between the branches. Notation-wise, the output is expressed as $\left(\bm{A}\bm{X}\right) \ast \left(\bm{\Lambda}\bm{X} \right)$, where $\ast$ denotes the elementwise product and $\bm{A}, \bm{\Lambda}$ are learnable matrices. 

We perform the following three modifications to the aforementioned block. Firstly, we add a shortcut connection to capture the first-degree interactions. Secondly, we decompose $\bm{\Lambda}$ into two matrices as $\bm{\Lambda} = \bm{B}\bm{D}$. This rank factorization of $\bm{\Lambda}$ enables us to capture low-rank representations in this branch. Thirdly, we add one matrix $\bm{C}$ to capture the linear interactions of the whole expression. Then, the complete layer is expressed as follows:
\begin{equation}
    \bm{Y} = \bm{C} \left[\left(\bm{A} \bm{X}\right) \ast \left(\bm{B} \bm{D}\bm{X} \right)+ \bm{A} \bm{X} \right]\;,
    \label{eq:poly_mixer_core_equation_mu_layer}
\end{equation}
where $\bm{A}\in\realnum^{m\times d}, \bm{B}\in\realnum^{m\times l}, \bm{C}\in\realnum^{o\times m}, \bm{D}\in\realnum^{l\times d}$ are learnable parameters and $*$ symbolizes an elementwise product. The block has two hyperparameters we can choose, i.e., the rank $m\in\naturalnum$ and the value $l\in\naturalnum$. In practice, we utilize a shrinkage ratio $\frac{m}{l}> 1$ to encourage different information flow in the branches. \cref{proposition:poly_mixer_multiplicative_interaction_per_layer}, the proof of which is in \cref{sec:poly_mixer_app_proof_proposition}, indicates the interactions captured.  
The schematic of this layer, called \layer, is depicted in \cref{appendix:PolyMLP}. 

\begin{proposition}
\label{proposition:poly_mixer_multiplicative_interaction_per_layer}
The \layer{} captures multiplicative interactions between elements of each token.
\end{proposition}

\textbf{\block}: The \block{}, which is the core operation of \modelnamePM{} architecture, stacks \layer s sequentially. Concretely, we connect two \layer s sequentially and add a a layer normalization~\citep{ba2016layernorm} between them. Additionally, we add a shortcut connection to skip the first block. The two \layer s are similar except for two differences: (a) a spatial shift operation \citep{yu2022s2} is added to the first block, (b) the hidden dimension is $m$ in the first \layer, and $m\cdot \epsilon$ in the second, where $\epsilon$ is an expansion factor. The \block{} is illustrated in \cref{fig:Poly Block}.

\subsection{Network Architecture}
\label{subsec:Network Architecture}

Our final architecture is composed of $N$ \block s. 
Each block captures up to $4\myth$ degree interactions, which results in the final architecture capturing up to $4^N$ interactions, with $N>10$ in practice\footnote{\rebuttal{A theoretical proof on the degree of expansion required to tackle a specific task is an interesting problem that is left as future work.}}. Beyond the \block, one critical detail in our final architecture is the patch embedding.

\textbf{Patch embeddings}: In deep learning architectures, the choice of patch size plays a crucial role in capturing fine-grained features. Smaller patches prove more effective in this regard. However, a trade-off arises, as reducing the patch size leads to an increase in floating-point operations, which scales approximately quadratically. To overcome this challenge, we introduce a novel scheme called pyramid patch embedding. This scheme operates by considering embeddings at smaller scales and subsequently extracting new patch embeddings on top. By adopting this approach, we can leverage the finer details of features without introducing additional parameters or computational complexity. Our experiments validate the effectiveness of incorporating multi-level patches, enhancing the overall performance. For the \imagenet{} dataset, we employ a two-level pyramid scheme, although it is worth noting that higher-resolution images may benefit from additional levels. We present the schematic and further details of this approach in \cref{appendix:pyramid embedding}.

\section{Experiments}
\label{sec:poly_mixer_experiments}
\begin{table}[tbp]
\centering
\caption{Configurations of different \modelnamePM{} models. We present four models: two variants with different parameter-size (Tiny, Small) and two variants with different hidden-size. The original model use the same hidden-size across the architecture. The Multi-stage model adopts a different hidden sizes across the architecture. The number in stages represents the number of blocks that each stage includes, while the hidden size list numbers correspond to each stage's hidden size. We provide further details along with a schematic in \cref{appendix:architecture}. We only use Multi-stage models for high-resolution image classification benchmarks, e.g., on \imagenet.}
\resizebox{\textwidth}{!}{
\begin{tabular}{ccccc}
\toprule
Specification     & Tiny                   & Small                   &Multi-stage-Tiny   &Multi-stage-Small \\ \midrule
Numbers of Blocks & 32                      & 45                     &32                &32\\
Embedding Method  & Pyramid Patch Embedding & Pyramid Patch Embedding & Pyramid Patch Embedding & Pyramid Patch Embedding \\
Hidden size       & 192                     & 320                     &{64,128,192,192} &{128,192,256,384} \\
Stages            & -                       & -                        &{4,8,12,10}    &{4,6,12,14} \\
Expansion Ratio   & 3                       & 3                       &3               & 3 \\
Shrinkage Ratio   & 4                       & 4                       & 8              & 8  \\
Parameters        & 14.0M                  & 56.8M                    & 10.3M           &  32.9M  \\ 
FLOPs             & 3.6                   & 11.2                      & 2.8             & 6.8                       \\ \bottomrule
\label{tab:specification}
\end{tabular}}
\end{table}

In this section, we conduct a thorough experimental validation of \modelnamePM. We conduct experiments on large-scale image classification in \cref{ssec:poly_mixer_experiment_imagenet}, fine-grained and small-scale image classification in \cref{ssec:poly_mixer_experiment_fine_grained}. In addition, we exhibit a unique advantage of models without activation functions to learn dynamic systems in scientific computing in \cref{ssec:neuralode experiment}. Lastly, we validate the robustness of our model to diverse perturbations in \cref{ssec:robust experiment}. 
We summarize four configurations of the proposed \modelnamePM{} in \cref{tab:specification} with different versions of \modelnamePM.  We present a schematic for our configuration in \cref{appendix:architecture}. Due to the limited space, we conduct additional experiments in the Appendix. Concretely, we experiment on semantic segmentation in \cref{appendix:semantic segmentation}, while we add additional ablations and experimental details on \cref{appendix:imagenet,appendix:add med,appendix:initial}.

\subsection{\imagenet{} Classification}
\label{ssec:poly_mixer_experiment_imagenet}
\imagenet, which is the standard benchmark for image classification, contains 1.2M images with 1,000 categories annotated. We consider a host of baseline models to compare the performance of \modelnamePM{}. Concretely, we include strong-performing polynomial models\footnote{\rebuttal{We note that we report results of $\Pi-$nets without activation functions~\citep{chrysos2020pinet} for a fair comparison. The models with activation functions are reported as `hubrid'.}}~\citep{chrysos2020pinet, chrysos2023regularization}, MLP-like models~\citep{tolstikhin2021mlp,touvron2022resmlp, yu2022s2}, models based on vanilla Transformer~\citep{vaswani2017Transformer,touvron2021DeiT} and several classic convolutional networks~\citep{bello2019attention,chen2018a2nets,simonyan2014vggnet,he2016resnet}.

\textbf{Training Setup}: We train our model using AdamW optimizer~\citep{loshchilov2017adamw}. We use a batch size of $448$ \rebuttal{per GPU} to fully leverage the memory capacity of the GPU. We use a linear warmup and cosine decay schedule learning rate, while the initial learning rate is 1e-4, linear increase to 1e-3 in 10 epochs and then gradually drops to 1e-5 in 300 epochs. We use label smoothing~\citep{szegedy2016labelsmooth}, standard data augmentation strategies, such as Cut-Mix~\citep{yun2019cutmix}, Mix-up~\citep{zhang2017mixup} and auto-augment~\citep{cubuk2018autoaugment}, which are used in similar methods~\citep{tolstikhin2021mlp,trockman2022patches,touvron2022resmlp}. Our data augmentation recipe follows the one used in MLP-Mixer~\citep{tolstikhin2021mlp}. We do not use any external data for training. We train our model using native PyTorch training on 4 NVIDIA A100 GPUs. 
\begin{table}[t]
  \caption{\imagenet{} classification accuracy (\%) on the validation set for different models. In the Activation Function column, G denotes Gelu, R denotes Relu, T denotes Tanh, and P denotes RPReLU. The models with \textcolor{blue}{$\dagger$} are special polynomial models with activation functions. \rebuttal{We mark model with up to 15M params in red color, models with 15-40M parameters with green color and model with over 40M parameters with blue.} %
  }
  \label{tab:imagenet-benchmark-updated}
  \centering
  \setlength{\tabcolsep}{0.6\tabcolsep}
  \resizebox{\textwidth}{!}{%
  \begin{tabular}{@{}lccccccc@{}}
        \toprule
    \multirow{2}{*}{} &  & \multicolumn{2}{c}{\textbf{Accuracy}} \\  \cmidrule(lr){3-4}
     & \textbf{Extra Data} & \textbf{Top-1(\%)} & \textbf{Top-5(\%)} &\textbf{FLOPs (B)} & \textbf{Params (M)}& {\textbf{Activation}}&{\textbf{Attention}} \\ \midrule
    \multicolumn{5}{c}{\hspace{7cm} \textbf{CNN-based}} \\ \midrule
    \rowcolor{small}
    ResNet-18~\citep{he2016resnet} &\ding{53} & 69.7 &89.0 & 1.8 & 11.0& R&\ding{53}\\
    \rowcolor{base}
    ResNet-50~\citep{he2016resnet}& \ding{53} & 77.2 & 92.9 &4.1 & 25.0 &R&\ding{53}\\
    \rowcolor{base}
    $A^2$Net~\citep{chen2018a2nets}& \ding{53} & 77.0 & 93.5 &31.3 & 33.4 &R&\ding{51}\\
    \rowcolor{big}
    AA-ResNet-152~\citep{bello2019attention}& \ding{53} & 79.1 & 94.6 &23.8 & 61.6 &R&\ding{51}\\
    \rowcolor{big}
    VGG-16~\citep{simonyan2014vggnet} & \ding{53}  & 71.5 & 92.7 &15.5 & 138.0 &R&\ding{53}\\
    \midrule
    \multicolumn{5}{c}{\hspace{7cm}\textbf{Transformer-based}} \\ \midrule
    \rowcolor{base}
     DeiT-S/16~\citep{touvron2021DeiT} & \ding{51} & 81.2 & - & 5.0 & 24.0 & G&\ding{51}\\
     \rowcolor{big}
    ViT-B/16~\citep{vaswani2017Transformer} & \ding{53} & 77.9 & - &55.5 & 86.0& G&\ding{51}\\
    \midrule
    \multicolumn{5}{c}{\hspace{7cm}\textbf{MLP-based}} \\ \midrule

    \rowcolor{small}
    BiMLP-S~\citep{xu2022bimlp} & \ding{53} & 70.0 &- & 1.21 & - &P&\ding{51}\\
    \rowcolor{small}
    BiMLP-B~\citep{xu2022bimlp} & \ding{53} & 72.7 &- & 1.21 & - &P&\ding{51}\\
    \rowcolor{small}
    ResMLP-12~\citep{touvron2022resmlp} & \ding{51}  & 76.6 & - &3.0 & 15.0& G&\ding{53}\\
    \rowcolor{base}
    Hire-MLP-Tiny~\citep{guo2022hire} & \ding{53} &79.7 &- & 2.1 & 18.0 &G&\ding{51}\\
     \rowcolor{base}
    CycleMLP-T~\citep{chen2021cyclemlp}& \ding{53} &81.3 &- & 4.4 & 28.8 &G&\ding{51}\\
    \rowcolor{base}
    ResMLP-24~\citep{touvron2022resmlp} & \ding{51}  & 79.4 &- & 6 & 30.0& G&\ding{53}\\
    \rowcolor{big}
    MLP-Mixer-B/16~\citep{tolstikhin2021mlp} & \ding{53} & 76.4 &- & 11.6 & 59.0& G&\ding{53}\\
    \rowcolor{big}
    MLP-Mixer-L/16~\citep{tolstikhin2021mlp} & \ding{53} & 71.8 &- & 44.6 & 207.0& G&\ding{53}\\
    \rowcolor{big}
    $S^2$MLP-Wide~\citep{yu2022s2} & \ding{53} &80.0 &94.8 & 14.0 & 71.0 &G&\ding{53}\\
    \rowcolor{big}
    $S^2$MLP-Deep~\citep{yu2022s2} & \ding{53} &80.7 &95.4 & 10.5 & 51.0 &G&\ding{53}\\
    \rowcolor{big}
    FF~\citep{melas2021FF} & \ding{53} & 74.9 & - &7.21 & 59.0 &G&\ding{53}\\ 
    \midrule
    \multicolumn{5}{c}{\hspace{7cm}\textbf{Polynomial-based}} \\ \midrule
    \rowcolor{small}
    $\Pi$-Nets~\citep{chrysos2020pinet}  & \ding{53} &65.2 &85.9 &1.9&12.3&None& \ding{53}\\ 
    \rowcolor{small}
    Hybrid $\Pi$-Nets~\citep{chrysos2020pinet}\textcolor{blue}{$\dagger$}  & \ding{53} & 70.7 &89.5 &1.9&11.9&R,T &\ding{53}\\ 
    \rowcolor{small}
    PDC-comp~\citep{chrysos2022augmenting}\textcolor{blue}{$\dagger$}&\ding{53} &69.8 & 89.9 &1.3&7.5&R,T&\ding{53}\\
    \rowcolor{small}
    PDC~\citep{chrysos2022augmenting}\textcolor{blue}{$\dagger$}&\ding{53} &71.0 & 89.9 &1.6&10.7&R,T&\ding{53}\\
    \rowcolor{small}
   \newmodelnamePI ~\citep{chrysos2023regularization} & \ding{53} & 70.2 &89.3 &1.9&12.3&None& \ding{53}\\
    \rowcolor{small}
   \modelnamedense~\citep{chrysos2023regularization}  & \ding{53} & 70.0 &89.4 &1.9&11.3&None&\ding{53} \\ 
    \rowcolor{small}
     \textbf{Multi-stage \modelnamePM-T (Ours) }& \ding{53} & \textbf{77.0} &\textbf{93.4} &2.8&10.3&None&\ding{53} \\ 
     \rowcolor{base}
     \textbf{Multi-stage \modelnamePM-S (Ours)}  & \ding{53} & \textbf{81.3} &\textbf{95.5} &6.8&32.9&None& \ding{53}\\ 
    \bottomrule
  \end{tabular}
  }
\end{table}

We exhibit the performance of the compared networks in \cref{tab:imagenet-benchmark-updated}. Notice that our smaller model achieves a \textbf{10\%} improvement over previous \shortpolyname. Our larger model further strengthens this performance gain and manages for the first time to close the gap between \shortpolyname{} and be on par with other recent models.

\subsection{Additional benchmarks in image recognition}
\label{ssec:poly_mixer_experiment_fine_grained}

\begin{table}[tb]
\begin{center}
\caption{Experimental validation of different polynomial architectures on smaller datasets. The best are marked in \textbf{bold}. 
     \modelnamedense{} perform favorably to most of the other baseline polynomial networks and the rest baselines. \modelnamePM{} outperforms all the baselines and all polynomial networks in two of the four datasets, and shares the best accuracy with \newmodelnamePI{} in another dataset.  
      }
      \label{tab:polymixer_experiment_recognition}
\resizebox{0.7\textwidth}{!}{%
\begin{tabular}{c|cccc}
\toprule
Model          & CIFAR10       & SVHN          & Oxford Flower & Tiny Imagenet \\ \hline
Resnet18       & 94.4          & 97.3          & 87.8          & 61.5          \\ \hline
MLP-Mixer      & 90.6          & 96.8          & 80.2          & 45.6          \\
Res-MLP        & 92.3          & 97.1          & 83.2          & 58.9          \\
$S^{2}$MLP-Deep-S   & 92.6          & 96.7          & 93.0          & 59.3          \\
$S^{2}$MLP-Wide-S   & 92.0          & 96.6          & 91.5          & 52.7          \\ \hline
Hybrid Pi-Nets & 94.4          & -             & 88.9          & 61.1          \\
$\Pi$-Nets     & 90.7          & 96.1          & 82.6          & 50.2          \\
PDC            & 90.9          & 96.0          & 88.5          & 45.2          \\
\newmodelnamePI     & 94.5          & 97.6          & 94.9          & 61.5          \\
\modelnamedense     & 94.7          & 97.5          & 94.1          & \textbf{61.8} \\
\modelnamePMS{}    & \textbf{94.8} & \textbf{97.6} & \textbf{95.0} & 61.5          \\ \hline
\end{tabular}}
\end{center}
\end{table}

Beyond \imagenet, we experiment with a number of additional benchmarks to further assess \modelnamePM. We use the standard datasets of CIFAR10 \citep{krizhevsky2009cifar10}, SVHN \citep{netzer2011svhn} and Tiny \imagenet{} \citep{le2015tinyimagenet} for image recognition. A fine-grained classification experiment on Oxford Flower102 \citep{oxfordflower} is conducted. Beyond the different distributions, those datasets offer insights into the performance of the proposed model in datasets with limited data\footnote{Previous MLP-based models have demonstrated weaker performance than CNN-based models in such datasets.  The original $S^{2}$MLP only reports two large models on these datasets. We use an open source $S^{2}$MLP code design two models with 12-14 M parameters for a fair comparison. Those models, noted as $S^{2}$MLP-Wide-S (hidden size 512,depth 8) and $S^{2}$MLP-Deep-S (hidden size 384,depth 12) are used for the comparisons.}.
The results in \cref{tab:polymixer_experiment_recognition} compare the proposed method with other strong-performing methods. \modelnamePM{} outperforms the compared methods with the second strongest being the \newmodelnamePI. MLP-models usually do not perform well in small datasets as indicated in \citet{tolstikhin2021mlp} due to their weaker inductive bias. Notably, \modelnamePM{}  still performs better than other CNN-based polynomial models.

\subsection{Poly Neural ODE Solver}
\label{ssec:neuralode experiment}
\rebuttal{An additional advantage of our method is the ability to model functions that have a polynomial form (e.g, in ordinary differential equations) in an interpretable manner}~\citep{fronk2023interpretable}. A neural ODE solver is a computational technique that models dynamic systems relying on ODEs using neural networks. \citet{chen2018neural,dupont2019augmented} illustrate the potential of neural ODEs to capture the behavior of complex systems that evolve over time. However, an evident limitation is that those models remain black-box which means the learned equation remains unknown after training. On the contrary, the proposed model can recover the equations behind the dynamic system and explicitly restore the symbolic representation.

We provide an experiment of a polynomial neural ODE approximating the Lotka-Volterra ODE model. This model captures the dynamics of predator-prey population within a biological system. The equations representing the Lotka-Volterra formula are in the form of $\frac{d x} {d t} = \alpha x - \beta x y$, $ \frac{d y}{d t} = -\delta y + \gamma x y\ $. Given $\alpha = 1.56$, $\beta =1.12$, $\delta  = 3.10$, $\gamma = 1.21$, the ground truth equation are presented below:
\begin{equation}
    \displaystyle\frac{d x} {d t} = 1.56 x - 1.12 x y\;,\;\;\;\;\;\;\;\;\;\;\;\;\;\;\;\;\;\;\;\;\;\;\;\;\;\;\;\;\;\;\; \displaystyle \frac{d y}{d t} = -3.10 y + 1.21 x y\;.
\end{equation}

We generate $N=100$ discrete points between time 0 and 10 for training. The training loss with epochs and predicted trajectory are shown in \cref{fig:lkequation}.
\begin{figure}[tbp]
\captionsetup{font=small}
  \centering
 \includegraphics[width=0.7\columnwidth]{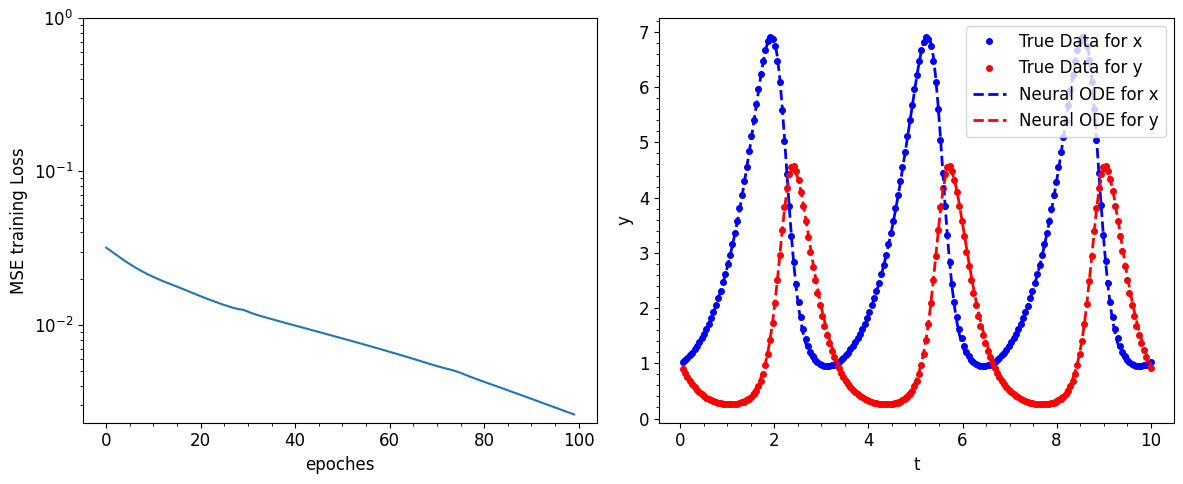}
   \caption{The training loss change with epochs trained(\textbf{Left}). The ground truth and model predicted trajectory. (\textbf{Right}) Our model achieves low loss in 20 epochs and successfully predicts real trajectory.}
\label{fig:lkequation}
 \end{figure}

Notice that our model can directly recover the right hand side of the Lotka-Volterra formula. Indeed, the learned model in our case recovers the following formula:
\begin{equation}
    \displaystyle\frac{d x} {d t} = 1.56 x - 1.12001 x y\;,\;\;\;\;\;\;\;\;\;\;\;\;\;\;\;\;\;\;\;\;\;\;\;\;\;\;\;\;\;\;\; \displaystyle \frac{d y}{d t} = -3.10 y + 1.21001 x y\;.
\end{equation}
Examining the results, it is evident that our Poly Neural ODE achieves high accuracy in recovering the ground truth formula and demonstrates rapid convergence. A more comprehensive comparison with NeuralODE and Augmented NeuralODE is conducted in \cref{appendix:NeuralODE Solver comparsion}.

\subsection{Robustness}
\label{ssec:robust experiment}

\begin{table}[tb]
\setlength{\tabcolsep}{0.225\tabcolsep}
\caption{Robustness on ImageNet-C \citep{hendrycks2018benchmarking}. The corruption error (and mean corruption error (mCE)) is used as the robustness metric, where a lower score indicates a better performance. The best performance per category is highlighted in bold. Notice that the proposed method is robust to a range of corruptions. In fact, in many categories, such as `Weather' and `Digital', the proposed model outperforms the compared models in all designated corruptions. 
}
\resizebox{\textwidth}{!}{%
\begin{tabular}{lc|ccc|cccc|cccc|cccc}
\hline
                  & \multicolumn{1}{l|}{} & \multicolumn{3}{c|}{Noise}                    & \multicolumn{4}{c|}{Blur}                                   & \multicolumn{4}{c|}{Weather}                     & \multicolumn{4}{c}{Digital}                                              \\
Network           & mCE(↓)                & Gauss         & Shot          & Impulse       & Defocus       & Glass         & Motion      & Zoom          & Snow & Frost       & Fog         & Bright        & Contrast      & Elastic       & Pixel         & \multicolumn{1}{c}{JPEG} \\ \hline
ResNet50          & 76.7                  & 79.8          & 81.6          & 82.6          & 74.7          & 88.6          & 78          & 79.9          & 77.8 & 74.8        & 66.1        & 56.6          & 71.4          & 84.7          & 76.9          & \multicolumn{1}{c}{76.8} \\ \hline
DeiT              & 54.6                  & 46.3          & 47.7          & 46.4          & 61.6          & \textbf{71.9} & 57.9        & 71.9          & 49.9 & 46.2        & 46          & 44.9          & 42.3          & 66.6          & 59.1          & 60.4                      \\
Swin              & 62.0                    & 52.2          & 53.7          & 53.6          & 67.9          & 78.6          & 64.1        & 75.3          & 55.8 & 52.8        & 51.3        & 48.1          & 45.1          & 75.7          & 76.3          & 79.1                      \\ \hline
MLP-Mixer         & 78.8                  & 80.9          & 82.6          & 84.2          & 86.9          & 92.1          & 79.1        & 93.6          & 78.3 & 67.4        & 64.6        & 59.5          & 57.1          & 90.5          & 72.7          & 92.2                      \\
ResMLP            & 66.0                    & 57.6          & 58.2          & 57.8          & 72.6          & 83.2          & 67.9        & 76.5          & 61.4 & 57.8        & 63.8        & 53.9          & 52.1          & 78.3          & 72.9          & 75.3                      \\
gMLP              & 64                    & 52.1          & 53.2          & 52.5          & 73.1          & 77.6          & 64.6        & 79.9          & 77.7 & 78.8        & 54.3        & 55.3          & 43.6          & 70.6          & 58.6          & 67.5                      \\
CycleMLP          & 53.7                  & \textbf{42.1} & \textbf{43.4} & \textbf{43.2} & 61.5          & 76.7          & 56.0          & 66.4          & 51.5 & 47.2        & 50.8        & 41.2          & 39.5          & 72.3          & 57.5          & 56.1                      \\ 
HireMLP  &51.9&52.4&55.3&55.3&60.3&\textbf{71.6}&59.6&57.8&54.3&52.5&43.5&29.1&42.2&54.9&49.7&40.2
\\ \hline
\newmodelnamePI &73.8& 84.0& 84.4&88.0&81.9&83.6&75.0&77.9&71.8&72.2&69.4&45.7&67.4&65.8&74.8&65.4\\
\modelnamePMB{} & \textbf{49.7}         & 51.3          & 52.3          & 53.0            & \textbf{57.8} & 72.1          & \textbf{55.0} & \textbf{60.6} & \textbf{50.4} & \textbf{46.0} & \textbf{42.0} & \textbf{27.6} & \textbf{38.1} & \textbf{53.7} & \textbf{47.7} & \textbf{39.0} \\     
\bottomrule
\end{tabular}}
\label{tab:robustness}
\end{table}
We further conduct experiments on ImageNet-C \citep{hendrycks2018benchmarking} to analyze the robustness of our model. ImageNet-C contains 75 types of corruptions, which makes it a good testbed for our evaluation. We follow the setting of CycleMLP \citep{chen2021cyclemlp} and compare \modelnamePM{} with existing models using the corruption error as the metric (lower value indicates a better performance). \cref{tab:robustness} illustrates that our models exhibits the strongest ability among recent models. Notice that in important categories, such as the `weather' or the `digital' category, \modelnamePM{} outperforms the compared methods in all types of corruption. Those categories can indeed be corruptions met in realistic scenaria, making us more confident on the robustness of the proposed model in such cases.

\subsection{Ablation Study}
\label{subsec:Ablation Analysis}

We conduct an ablation study using \abdataset, a subset of \imagenet{} with $100$ classes selected from the original ones. Previous studies have reported that \abdataset{} serves as a representative subset of \imagenet~\citep{yan2021der,yu2021s2MLPV2,douillard2021dytox}. Therefore, \abdataset{} enables us to make efficient use of computational resources while still obtaining representative results for the self-evaluation of the model. 
In every experiment below, we assess one module or hyperparameter, while we assume that the rest remain the same as in the final model.

\textbf{Module ablation}: In this study, we investigate the influence of various modules on the final network performance by removing and reassembling modules within the block. Please note that the following blocks do not contain any activation functions. In \cref{tab:layer number ablation}, we report the result using a hidden size of $384$ and depth $16$ on the \abdataset. We observe that spatial shift contributes to the performance improvement in our model but it is not crucial. The network still achieves high accuracy even in the absence of this module. At the same time, the removal of the \layer{} from the model leads to a significant drop in performance. %
The results validate that the spatial shift module alone cannot replace the proposed multilinear operations. 
\begin{table}[tb]
\centering
\small
\caption{The influence of the different module in Block. Based on depth $16$, hidden size $384$ model on ImageNet100.
We incorporate spatial shift since this operation is fully linear and brings a  performance gain. However, the primary performance improvement comes from the \layer{} we proposed as shown in the following table. Notice all models mentioned below are without activation functions.}
\resizebox{0.8\textwidth}{!}{%
\begin{tabular}{@{}ccccc@{}}
\toprule
\textbf{Block Type} & \textbf{Layer 1 } & \textbf{Layer 2 } & \textbf{Paras. (M)} & \textbf{Top-1 Acc(\%)} \\
\hline
\block & \layer +Spatial-Shift & \layer  & 1.70 & $82.94\%$\\
\block* & \layer  & \layer  & 1.70 & $81.50\%$\\
Mix Block & \layer +Spatial-Shift & MLP Layer & 1.26 & $67.61\%$ \\
Linear Block & MLP Layer+Spatial-Shift & MLP Layer & 1.21 & $55.11\%$ \\
\hline
\end{tabular}
}
\label{tab:layer number ablation}
\end{table}

\textbf{Patch Embedding}: \label{subsec:patch embed text}
In this study, we validate the effectiveness of our proposed pyramid patch embedding. We replace the input embedding layer of a model with depth $16$ and hidden size $384$ on the \abdataset{} dataset. As depicted in \cref{tab:patch embed ablation}, the model utilizing multi-level embedding achieves an Multi-stage performance with a Top-1 accuracy of $82.94\%$. The model utilizing a single-level embedding achieves a Top-1 accuracy of $81.78\%$, which is slightly lower than the multi-level embedding approach. For normal patch embedding, the computational and parameter overhead increases quadratically with a decrease in patch size. The result highlights that our approach significantly improves performance while adding only a small number of parameters and computational overhead. We conduct the same study for MLP-Mixer in \cref{appendix:pyramid embedding} with similar outcomes, indicating the efficacy of our core \block{} independently of the patch embedding module. 

 \begin{table}[tb]
\centering
\small
\caption{The influence of the patch embedding approach on a model with depth $16$, hidden size $384$ trained on \abdataset. Notice that the pyramid patch embedding can improve the performance with a minor difference in the parameters and maintain small FLOPs when using small patch.}
\begin{tabular}{@{}ccccc@{}}
\toprule
\textbf{Method} & \textbf{Patch size } & \textbf{Top-1 Acc(\%)} & \textbf{Paras (M)} & \textbf{FLOPs} \\
\hline
Patch Embedding & 7 & 83.08  & 27.77 & 28.14\\
Patch Embedding & 14 & 79.54  & 27.94 & 7.07\\
Pyramid Patch Embedding & 7 & 82.04 & 28.02 & 7.23\\
Pyramid Patch Embedding & 14,7 &82.94 & 28.54 & 7.28 \\
\hline
\end{tabular}
\label{tab:patch embed ablation}
\end{table}

\textbf{Hidden size}: We evaluate the impact of different hidden size $m$. We fix the depth of the network to $16$ and vary the hidden size. 
We observe a sharp decrease when MLP-Mixer uses a small hidden size under $128$. The results in \cref{tab:hidden size} validate that our proposed method is more robust in hidden size change compared to the normal MLP layer. 
\begin{table}[tb]
    \centering
    \caption{ The influence of the hidden size, $m$.The models result are based on depth $16$ model with different hidden size, the number in the second and their row indicates Top-1 Accuracy(\%) on \abdataset. The MLP-Mixer shows a sharp decline when hidden size change from $128$ to $96$. }
    \begin{tabular}{c c c c c c c}
    \hline
        $m$ & 96 & 128 & 192 & 256 & 384 & 512  \\ \hline
        MLP-Mixer & 48.3\% & 70.05\% & 72.39\% & 76.88\% & 78.24\% & 77.93\%  \\ \hline
        \modelnamePM{}& 67.5\% & 75.82\% & 79.3\% & 81.94\% & 82.94\% & 82.24\% \\ \hline
    \end{tabular}
\label{tab:hidden size}
\end{table}

\textbf{Depth}: The backbone of our proposed model, \modelnamePM{}, consists of $N$  blocks. We evaluate the influence of the number of blocks (depth $N$) on Top-1 accuracy and parameter numbers, using a model with a fixed hidden size and shrinkage ratio. 
We observe that for our network, the depth holds a more significance role than the hidden size. Our experiments result in \cref{tab:depth} validate our theoretical conjecture, as increasing depth yields a more pronounced improvement in performance.

\begin{table}[tb]
 \begin{minipage}{0.45\linewidth}
    \centering
    \caption{ The influence of the depth, $N$. The model is based on depth $16$, hidden size $384$ and is trained on \abdataset. }
    \begin{tabular}{@{}lcccc@{}}
    \toprule
    \textbf{$N$} & \textbf{Top-1(\%)} & \textbf{Top-5(\%)} & \textbf{Paras (M)} \\
    \midrule
    1 & 44.25 & 71.18 & 2.96 \\
    3 & 70.58 & 78.30 & 6.37 \\
    6 & 78.30 & 93.20 & 11.49 \\
    12 & 80.03 & 94.12 & 21.72 \\
    16 & 82.94 & 95.04 & 28.54 \\
    \bottomrule
    \end{tabular}
    \label{tab:depth}
\end{minipage}
\hspace{1cm}
\begin{minipage}{0.45\linewidth}
    \centering
    \caption{ The influence of the shrinkage ratio, $r$. The model is based on depth $16$, hidden size $m=384$,  trained on \abdataset. }
    \begin{tabular}{@{}lcccc@{}}
   \toprule
    \textbf{$r$} & \textbf{Top-1(\%)} & \textbf{Top-5(\%)} & \textbf{Paras (M)} \\
    \midrule
    1 & 83.18 & 94.86 & 52.98 \\
    2 & 83.88 & 95.54 & 36.46 \\
    4 & 82.94 & 95.04 & 28.54 \\
    8 & 82.28 & 95.12 & 24.06 \\
    16 & 82.34 & 94.62 & 21.99 \\
    \bottomrule
    \end{tabular}
    \label{tab:shrinkage}
    \end{minipage}
\end{table}

\textbf{Shrinkage ratio}:  As mentioned in \cref{subsec:Polymlp} the shrinkage ratio is defined as $r=\frac{m}{l}$. The results of different shrinkage ratios based on a hidden size of $m=384$ 
and a depth of $16$ in our model are presented in \cref{tab:shrinkage}. We observe that: i) A larger shrinkage ratio effectively reduces the number of parameters while having a relatively minor impact on performance; ii) a lower effective rank in the weights is sufficient for performance gain.

\section{Conclusion}

In this work, we introduce a model, called \modelnamePM{} that  expresses the output as a polynomial expansion of the input elements. \modelnamePM{} leverages \emph{solely} linear and multilinear operations, which avoids the requirement for activation functions. At the core of our model lies the \layer, which captures multiplicative interactions inside the token elements (i.e. input). Through a comprehensive evaluation, we demonstrate that \modelnamePM{} surpasses recent polynomial networks, showcasing performance levels outperforms modern transformers models across a range of challenging benchmarks in image recognition. We anticipate that our work will further encourage the community to reconsider the role of activation functions and actively explore alternative classes of functions that do not require them. Lastly, we encourage the community to extend our illustration of the polynomial ODE solver in order to tackle scientific applications with \shortpolyname.

\textbf{Limitation}: 
A theoretical characterization of the polynomial expansions that can be expressed with \modelnamePM{} remains elusive. In our future work, we will conduct further theoretical analysis of our model.
We believe that such an analysis would further shed light on the inductive bias of the block and its potential outside of image recognition.

\newpage
\section*{Reproducibility Statement}
Throughout this study, we exclusively utilize publicly accessible benchmarks, guaranteeing that other researchers can replicate our experiments. Additionally, we provide comprehensive information about the hyperparameters employed in our study and strive to offer in-depth explanations of all the techniques employed. Our plan is to make the source code of our model open source once our work gets accepted. 

\section*{Acknowledgments}
We are thankful to Dr. Giorgos Bouritsas and the ICLR reviewers for their feedback and constructive comments. We are also thankful to Colby Fronk for his help in the NeuralODE symbolic representation restoration. We thank Zulip\footnote{\url{https://zulip.com}} for their project organization tool. This work was supported by the Hasler Foundation Program: Hasler Responsible AI (project number 21043). Research was sponsored by the Army Research Office and was accomplished under Grant Number W911NF-24-1-0048. This work was supported by the Swiss National Science Foundation (SNSF) under grant number 200021\_205011.
\bibliography{references}

\begin{thebibliography}{66}
\providecommand{\natexlab}[1]{#1}
\providecommand{\url}[1]{\texttt{#1}}
\expandafter\ifx\csname urlstyle\endcsname\relax
  \providecommand{\doi}[1]{doi: #1}\else
  \providecommand{\doi}{doi: \begingroup \urlstyle{rm}\Url}\fi

\bibitem[Ba et~al.(2016)Ba, Kiros, and Hinton]{ba2016layernorm}
Jimmy~Lei Ba, Jamie~Ryan Kiros, and Geoffrey~E Hinton.
\newblock Layer normalization.
\newblock \emph{arXiv preprint arXiv:1607.06450}, 2016.

\bibitem[Babiloni et~al.(2021)Babiloni, Marras, Kokkinos, Deng, Chrysos, and Zafeiriou]{babiloni2021poly}
Francesca Babiloni, Ioannis Marras, Filippos Kokkinos, Jiankang Deng, Grigorios Chrysos, and Stefanos Zafeiriou.
\newblock Poly-nl: Linear complexity non-local layers with polynomials.
\newblock \emph{International Conference on Computer Vision (ICCV)}, 2021.

\bibitem[Bello et~al.(2019)Bello, Zoph, Vaswani, Shlens, and Le]{bello2019attention}
Irwan Bello, Barret Zoph, Ashish Vaswani, Jonathon Shlens, and Quoc~V Le.
\newblock Attention augmented convolutional networks.
\newblock In \emph{Conference on Computer Vision and Pattern Recognition (CVPR)}, pp.\  3286--3295, 2019.

\bibitem[Brakerski et~al.(2014)Brakerski, Gentry, and Vaikuntanathan]{brakerski2014leveled}
Zvika Brakerski, Craig Gentry, and Vinod Vaikuntanathan.
\newblock (leveled) fully homomorphic encryption without bootstrapping.
\newblock \emph{ACM Transactions on Computation Theory (TOCT)}, 6\penalty0 (3):\penalty0 1--36, 2014.

\bibitem[Caruana et~al.(2015)Caruana, Lou, Gehrke, Koch, Sturm, and Elhadad]{caruana2015intelligible}
Rich Caruana, Yin Lou, Johannes Gehrke, Paul Koch, Marc Sturm, and Noemie Elhadad.
\newblock Intelligible models for healthcare: Predicting pneumonia risk and hospital 30-day readmission.
\newblock In \emph{Proceedings of the 21th ACM SIGKDD international conference on knowledge discovery and data mining}, pp.\  1721--1730, 2015.

\bibitem[Chen et~al.(2018{\natexlab{a}})Chen, Rubanova, Bettencourt, and Duvenaud]{chen2018neural}
Ricky~TQ Chen, Yulia Rubanova, Jesse Bettencourt, and David~K Duvenaud.
\newblock Neural ordinary differential equations.
\newblock \emph{Advances in neural information processing systems (NeurIPS)}, 31, 2018{\natexlab{a}}.

\bibitem[Chen et~al.(2022)Chen, Xie, Ge, Chen, Liang, and Luo]{chen2021cyclemlp}
Shoufa Chen, Enze Xie, Chongjian Ge, Runjian Chen, Ding Liang, and Ping Luo.
\newblock Cyclemlp: A mlp-like architecture for dense prediction.
\newblock \emph{International Conference on Learning Representations (ICLR)}, 2022.

\bibitem[Chen et~al.(2018{\natexlab{b}})Chen, Kalantidis, Li, Yan, and Feng]{chen2018a2nets}
Yunpeng Chen, Yannis Kalantidis, Jianshu Li, Shuicheng Yan, and Jiashi Feng.
\newblock A\^{} 2-nets: Double attention networks.
\newblock \emph{Advances in neural information processing systems (NeurIPS)}, 31, 2018{\natexlab{b}}.

\bibitem[Chrysos et~al.(2021)Chrysos, Georgopoulos, and Panagakis]{chrysos2021conditional}
Grigorios Chrysos, Markos Georgopoulos, and Yannis Panagakis.
\newblock Conditional generation using polynomial expansions.
\newblock \emph{Advances in Neural Information Processing Systems}, 34:\penalty0 28390--28404, 2021.

\bibitem[Chrysos et~al.(2020)Chrysos, Moschoglou, Bouritsas, Panagakis, Deng, and Zafeiriou]{chrysos2020pinet}
Grigorios~G Chrysos, Stylianos Moschoglou, Giorgos Bouritsas, Yannis Panagakis, Jiankang Deng, and Stefanos Zafeiriou.
\newblock P-nets: Deep polynomial neural networks.
\newblock In \emph{Conference on Computer Vision and Pattern Recognition (CVPR)}, pp.\  7325--7335, 2020.

\bibitem[Chrysos et~al.(2022)Chrysos, Georgopoulos, Deng, Kossaifi, Panagakis, and Anandkumar]{chrysos2022augmenting}
Grigorios~G Chrysos, Markos Georgopoulos, Jiankang Deng, Jean Kossaifi, Yannis Panagakis, and Anima Anandkumar.
\newblock Augmenting deep classifiers with polynomial neural networks.
\newblock In \emph{European Conference on Computer Vision (ECCV)}, pp.\  692--716. Springer, 2022.

\bibitem[Chrysos et~al.(2023)Chrysos, Wang, Deng, and Cevher]{chrysos2023regularization}
Grigorios~G Chrysos, Bohan Wang, Jiankang Deng, and Volkan Cevher.
\newblock Regularization of polynomial networks for image recognition.
\newblock \emph{Conference on Computer Vision and Pattern Recognition (CVPR)}, 2023.

\bibitem[Cubuk et~al.(2019)Cubuk, Zoph, Mane, Vasudevan, and Le]{cubuk2018autoaugment}
Ekin~D Cubuk, Barret Zoph, Dandelion Mane, Vijay Vasudevan, and Quoc~V Le.
\newblock Autoaugment: Learning augmentation policies from data.
\newblock \emph{Conference on Computer Vision and Pattern Recognition (CVPR)}, 2019.

\bibitem[Dosovitskiy et~al.(2020)Dosovitskiy, Beyer, Kolesnikov, Weissenborn, Zhai, Unterthiner, Dehghani, Minderer, Heigold, Gelly, et~al.]{dosovitskiy2020ViT}
Alexey Dosovitskiy, Lucas Beyer, Alexander Kolesnikov, Dirk Weissenborn, Xiaohua Zhai, Thomas Unterthiner, Mostafa Dehghani, Matthias Minderer, Georg Heigold, Sylvain Gelly, et~al.
\newblock An image is worth 16x16 words: Transformers for image recognition at scale.
\newblock \emph{International Conference on Learning Representations (ICLR)}, 2020.

\bibitem[Douillard et~al.(2022)Douillard, Ram\'e, Couairon, and Cord]{douillard2021dytox}
Arthur Douillard, Alexandre Ram\'e, Guillaume Couairon, and Matthieu Cord.
\newblock Dytox: Transformers for continual learning with dynamic token expansion.
\newblock In \emph{Conference on Computer Vision and Pattern Recognition (CVPR)}, 2022.

\bibitem[Dubey et~al.(2022)Dubey, Radenovic, and Mahajan]{dubey2022scalable}
Abhimanyu Dubey, Filip Radenovic, and Dhruv Mahajan.
\newblock Scalable interpretability via polynomials.
\newblock \emph{Advances in neural information processing systems (NeurIPS)}, 2022.

\bibitem[Dupont et~al.(2019)Dupont, Doucet, and Teh]{dupont2019augmented}
Emilien Dupont, Arnaud Doucet, and Yee~Whye Teh.
\newblock Augmented neural odes.
\newblock \emph{Advances in neural information processing systems (NeurIPS)}, 32, 2019.

\bibitem[Fronk \& Petzold(2023)Fronk and Petzold]{fronk2023interpretable}
Colby Fronk and Linda Petzold.
\newblock Interpretable polynomial neural ordinary differential equations.
\newblock \emph{Chaos: An Interdisciplinary Journal of Nonlinear Science}, 33\penalty0 (4), 2023.

\bibitem[Georgopoulos et~al.(2020)Georgopoulos, Chrysos, Pantic, and Panagakis]{georgopoulos2020multilinear}
Markos Georgopoulos, Grigorios Chrysos, Maja Pantic, and Yannis Panagakis.
\newblock Multilinear latent conditioning for generating unseen attribute combinations.
\newblock In \emph{International Conference on Machine Learning}, pp.\  3442--3451. PMLR, 2020.

\bibitem[Georgopoulos et~al.(2021)Georgopoulos, Oldfield, Nicolaou, Panagakis, and Pantic]{georgopoulos2021mitigating}
Markos Georgopoulos, James Oldfield, Mihalis~A Nicolaou, Yannis Panagakis, and Maja Pantic.
\newblock Mitigating demographic bias in facial datasets with style-based multi-attribute transfer.
\newblock \emph{International Journal of Computer Vision}, 129\penalty0 (7):\penalty0 2288--2307, 2021.

\bibitem[Glorot \& Bengio(2010)Glorot and Bengio]{glorot2010xavier}
Xavier Glorot and Yoshua Bengio.
\newblock Understanding the difficulty of training deep feedforward neural networks.
\newblock In \emph{International Conference on Artificial Intelligence and Statistics (AISTATS)}, pp.\  249--256, 2010.

\bibitem[Guo et~al.(2022)Guo, Tang, Han, Chen, Wu, Xu, Xu, and Wang]{guo2022hire}
Jianyuan Guo, Yehui Tang, Kai Han, Xinghao Chen, Han Wu, Chao Xu, Chang Xu, and Yunhe Wang.
\newblock Hire-mlp: Vision mlp via hierarchical rearrangement.
\newblock In \emph{Conference on Computer Vision and Pattern Recognition (CVPR)}, pp.\  826--836, 2022.

\bibitem[He et~al.(2015)He, Zhang, Ren, and Sun]{he2015delving}
Kaiming He, Xiangyu Zhang, Shaoqing Ren, and Jian Sun.
\newblock Delving deep into rectifiers: Surpassing human-level performance on imagenet classification.
\newblock In \emph{International Conference on Computer Vision (ICCV)}, pp.\  1026--1034, 2015.

\bibitem[He et~al.(2016)He, Zhang, Ren, and Sun]{he2016resnet}
Kaiming He, Xiangyu Zhang, Shaoqing Ren, and Jian Sun.
\newblock Deep residual learning for image recognition.
\newblock In \emph{Conference on Computer Vision and Pattern Recognition (CVPR)}, pp.\  770--778, 2016.

\bibitem[Hendrycks \& Dietterich(2019)Hendrycks and Dietterich]{hendrycks2018benchmarking}
Dan Hendrycks and Thomas~G Dietterich.
\newblock Benchmarking neural network robustness to common corruptions and surface variations.
\newblock \emph{International Conference on Learning Representations (ICLR)}, 2019.

\bibitem[Hou et~al.(2022)Hou, Jiang, Yuan, Cheng, Yan, and Feng]{hou2022ViP}
Qibin Hou, Zihang Jiang, Li~Yuan, Ming-Ming Cheng, Shuicheng Yan, and Jiashi Feng.
\newblock Vision permutator: A permutable mlp-like architecture for visual recognition.
\newblock \emph{IEEE Transactions on Pattern Analysis and Machine Intelligence}, 45\penalty0 (1):\penalty0 1328--1334, 2022.

\bibitem[Hu et~al.(2018)Hu, Shen, and Sun]{hu2018squeeze}
Jie Hu, Li~Shen, and Gang Sun.
\newblock Squeeze-and-excitation networks.
\newblock In \emph{Conference on Computer Vision and Pattern Recognition (CVPR)}, pp.\  7132--7141, 2018.

\bibitem[Ivakhnenko(1971)]{ivakhnenko1971polynomial}
Alexey~Grigorevich Ivakhnenko.
\newblock Polynomial theory of complex systems.
\newblock \emph{IEEE transactions on Systems, Man, and Cybernetics}, 1\penalty0 (4):\penalty0 364--378, 1971.

\bibitem[Kolda \& Bader(2009)Kolda and Bader]{kolda2009tensor}
Tamara~G Kolda and Brett~W Bader.
\newblock Tensor decompositions and applications.
\newblock \emph{SIAM review}, 51\penalty0 (3):\penalty0 455--500, 2009.

\bibitem[Krizhevsky et~al.(2009)Krizhevsky, Hinton, et~al.]{krizhevsky2009cifar10}
Alex Krizhevsky, Geoffrey Hinton, et~al.
\newblock Learning multiple layers of features from tiny images.
\newblock \emph{University of Toronto}, 2009.

\bibitem[Le \& Yang(2015)Le and Yang]{le2015tinyimagenet}
Ya~Le and Xuan Yang.
\newblock Tiny imagenet visual recognition challenge.
\newblock \emph{CS 231N}, 7\penalty0 (7):\penalty0 3, 2015.

\bibitem[LeCun et~al.(2002)LeCun, Bottou, Orr, and M{\"u}ller]{lecun2002efficient}
Yann LeCun, L{\'e}on Bottou, Genevieve~B Orr, and Klaus-Robert M{\"u}ller.
\newblock Efficient backprop.
\newblock In \emph{Neural networks: Tricks of the trade}, pp.\  9--50. Springer, 2002.

\bibitem[Li(2003)]{li2003spsnn}
Chien-Kuo Li.
\newblock A sigma-pi-sigma neural network (spsnn).
\newblock \emph{Neural Processing Letters}, 17:\penalty0 1--19, 2003.

\bibitem[Li et~al.(2019)Li, Wang, Hu, and Yang]{li2019selective}
Xiang Li, Wenhai Wang, Xiaolin Hu, and Jian Yang.
\newblock Selective kernel networks.
\newblock In \emph{Conference on Computer Vision and Pattern Recognition (CVPR)}, pp.\  510--519, 2019.

\bibitem[Lin et~al.(2017)Lin, Doll{\'a}r, Girshick, He, Hariharan, and Belongie]{lin2017fpn}
Tsung-Yi Lin, Piotr Doll{\'a}r, Ross Girshick, Kaiming He, Bharath Hariharan, and Serge Belongie.
\newblock Feature pyramid networks for object detection.
\newblock In \emph{Conference on Computer Vision and Pattern Recognition (CVPR)}, pp.\  2117--2125, 2017.

\bibitem[Liu et~al.(2021)Liu, Dai, So, and Le]{liu2021pay}
Hanxiao Liu, Zihang Dai, David So, and Quoc~V Le.
\newblock Pay attention to mlps.
\newblock \emph{Advances in neural information processing systems (NeurIPS)}, 34:\penalty0 9204--9215, 2021.

\bibitem[Long et~al.(2015)Long, Shelhamer, and Darrell]{long2015fully}
Jonathan Long, Evan Shelhamer, and Trevor Darrell.
\newblock Fully convolutional networks for semantic segmentation.
\newblock In \emph{Proceedings of the IEEE conference on computer vision and pattern recognition}, pp.\  3431--3440, 2015.

\bibitem[Loshchilov \& Hutter(2019)Loshchilov and Hutter]{loshchilov2017adamw}
Ilya Loshchilov and Frank Hutter.
\newblock Decoupled weight decay regularization.
\newblock \emph{International Conference on Learning Representations (ICLR)}, 2019.

\bibitem[Lu \& Weng(2007)Lu and Weng]{lu2007survey}
Dengsheng Lu and Qihao Weng.
\newblock A survey of image classification methods and techniques for improving classification performance.
\newblock \emph{International journal of Remote sensing}, 28\penalty0 (5):\penalty0 823--870, 2007.

\bibitem[Martens et~al.(2010)]{martens2010sparse}
James Martens et~al.
\newblock Deep learning via hessian-free optimization.
\newblock In \emph{International Conference on Machine Learning (ICML)}, volume~27, pp.\  735--742, 2010.

\bibitem[Melas-Kyriazi(2021)]{melas2021FF}
Luke Melas-Kyriazi.
\newblock Do you even need attention? a stack of feed-forward layers does surprisingly well on imagenet.
\newblock \emph{arXiv preprint arXiv:2105.02723}, 2021.

\bibitem[Netzer et~al.(2011)Netzer, Wang, Coates, Bissacco, Wu, and Ng]{netzer2011svhn}
Yuval Netzer, Tao Wang, Adam Coates, Alessandro Bissacco, Bo~Wu, and Andrew~Y Ng.
\newblock Reading digits in natural images with unsupervised feature learning.
\newblock \emph{Advances in neural information processing systems (NeurIPS)}, 2011.

\bibitem[Nilsback \& Zisserman(2008)Nilsback and Zisserman]{oxfordflower}
Maria-Elena Nilsback and Andrew Zisserman.
\newblock Automated flower classification over a large number of classes.
\newblock In \emph{2008 Sixth Indian Conference on Computer Vision, Graphics and Image Processing}, pp.\  722--729, 2008.
\newblock \doi{10.1109/ICVGIP.2008.47}.

\bibitem[Peng et~al.(2022)Peng, Qiang, and Wu]{peng2022survey}
Luzhou Peng, Bowen Qiang, and Jiacheng Wu.
\newblock A survey: Image classification models based on convolutional neural networks.
\newblock In \emph{2022 14th International Conference on Computer Research and Development (ICCRD)}, pp.\  291--298. IEEE, 2022.

\bibitem[Plested \& Gedeon(2022)Plested and Gedeon]{plested2022deep}
Jo~Plested and Tom Gedeon.
\newblock Deep transfer learning for image classification: a survey, 2022.

\bibitem[Ronneberger et~al.(2015)Ronneberger, Fischer, and Brox]{ronneberger2015unet}
Olaf Ronneberger, Philipp Fischer, and Thomas Brox.
\newblock U-net: Convolutional networks for biomedical image segmentation.
\newblock In \emph{Medical Image Computing and Computer-Assisted Intervention--MICCAI 2015: 18th International Conference, Munich, Germany, October 5-9, 2015, Proceedings, Part III 18}, pp.\  234--241. Springer, 2015.

\bibitem[Shin \& Ghosh(1991)Shin and Ghosh]{shin1991pi-sigma}
Yoan Shin and Joydeep Ghosh.
\newblock The pi-sigma network: An efficient higher-order neural network for pattern classification and function approximation.
\newblock In \emph{IJCNN-91-Seattle international joint conference on neural networks}, volume~1, pp.\  13--18. IEEE, 1991.

\bibitem[Simonyan \& Zisserman(2015)Simonyan and Zisserman]{simonyan2014vggnet}
Karen Simonyan and Andrew Zisserman.
\newblock Very deep convolutional networks for large-scale image recognition.
\newblock \emph{International Conference on Learning Representations (ICLR)}, 2015.

\bibitem[Szegedy et~al.(2016)Szegedy, Vanhoucke, Ioffe, Shlens, and Wojna]{szegedy2016labelsmooth}
Christian Szegedy, Vincent Vanhoucke, Sergey Ioffe, Jon Shlens, and Zbigniew Wojna.
\newblock Rethinking the inception architecture for computer vision.
\newblock In \emph{Conference on Computer Vision and Pattern Recognition (CVPR)}, pp.\  2818--2826, 2016.

\bibitem[Tolstikhin et~al.(2021)Tolstikhin, Houlsby, Kolesnikov, Beyer, Zhai, Unterthiner, Yung, Steiner, Keysers, Uszkoreit, et~al.]{tolstikhin2021mlp}
Ilya~O Tolstikhin, Neil Houlsby, Alexander Kolesnikov, Lucas Beyer, Xiaohua Zhai, Thomas Unterthiner, Jessica Yung, Andreas Steiner, Daniel Keysers, Jakob Uszkoreit, et~al.
\newblock Mlp-mixer: An all-mlp architecture for vision.
\newblock \emph{Advances in neural information processing systems (NeurIPS)}, 34:\penalty0 24261--24272, 2021.

\bibitem[Touvron et~al.(2021)Touvron, Cord, Douze, Massa, Sablayrolles, and J{\'e}gou]{touvron2021DeiT}
Hugo Touvron, Matthieu Cord, Matthijs Douze, Francisco Massa, Alexandre Sablayrolles, and Herv{\'e} J{\'e}gou.
\newblock Training data-efficient image transformers \& distillation through attention.
\newblock In \emph{International Conference on Machine Learning (ICML)}, pp.\  10347--10357. PMLR, 2021.

\bibitem[Touvron et~al.(2022)Touvron, Bojanowski, Caron, Cord, El-Nouby, Grave, Izacard, Joulin, Synnaeve, Verbeek, et~al.]{touvron2022resmlp}
Hugo Touvron, Piotr Bojanowski, Mathilde Caron, Matthieu Cord, Alaaeldin El-Nouby, Edouard Grave, Gautier Izacard, Armand Joulin, Gabriel Synnaeve, Jakob Verbeek, et~al.
\newblock Resmlp: Feedforward networks for image classification with data-efficient training.
\newblock \emph{{IEEE} Transactions on Pattern Analysis and Machine Intelligence (T-PAMI)}, 2022.

\bibitem[Trockman \& Kolter(2023)Trockman and Kolter]{trockman2022patches}
Asher Trockman and J~Zico Kolter.
\newblock Patches are all you need?
\newblock \emph{Transactions on Machine Learning Research}, 2023.
\newblock ISSN 2835-8856.

\bibitem[Vaswani et~al.(2017)Vaswani, Shazeer, Parmar, Uszkoreit, Jones, Gomez, Kaiser, and Polosukhin]{vaswani2017Transformer}
Ashish Vaswani, Noam Shazeer, Niki Parmar, Jakob Uszkoreit, Llion Jones, Aidan~N Gomez, {\L}ukasz Kaiser, and Illia Polosukhin.
\newblock Attention is all you need.
\newblock \emph{Advances in neural information processing systems (NeurIPS)}, 30, 2017.

\bibitem[Wang et~al.(2021)Wang, Xie, Li, Fan, Song, Liang, Lu, Luo, and Shao]{wang2021PVT}
Wenhai Wang, Enze Xie, Xiang Li, Deng-Ping Fan, Kaitao Song, Ding Liang, Tong Lu, Ping Luo, and Ling Shao.
\newblock Pyramid vision transformer: A versatile backbone for dense prediction without convolutions.
\newblock In \emph{Conference on Computer Vision and Pattern Recognition (CVPR)}, pp.\  568--578, 2021.

\bibitem[Wang et~al.(2018)Wang, Girshick, Gupta, and He]{wang2018non}
Xiaolong Wang, Ross Girshick, Abhinav Gupta, and Kaiming He.
\newblock Non-local neural networks.
\newblock In \emph{Conference on Computer Vision and Pattern Recognition (CVPR)}, pp.\  7794--7803, 2018.

\bibitem[Xie et~al.(2021)Xie, Wang, Yu, Anandkumar, Alvarez, and Luo]{xie2021segformer}
Enze Xie, Wenhai Wang, Zhiding Yu, Anima Anandkumar, Jose~M Alvarez, and Ping Luo.
\newblock Segformer: Simple and efficient design for semantic segmentation with transformers.
\newblock \emph{Advances in Neural Information Processing Systems}, 34:\penalty0 12077--12090, 2021.

\bibitem[Xu et~al.(2022)Xu, Chen, and Wang]{xu2022bimlp}
Yixing Xu, Xinghao Chen, and Yunhe Wang.
\newblock Bimlp: Compact binary architectures for vision multi-layer perceptrons.
\newblock \emph{Advances in neural information processing systems (NeurIPS)}, 2022.

\bibitem[Yan et~al.(2021)Yan, Xie, and He]{yan2021der}
Shipeng Yan, Jiangwei Xie, and Xuming He.
\newblock Der: Dynamically expandable representation for class incremental learning.
\newblock In \emph{Conference on Computer Vision and Pattern Recognition (CVPR)}, pp.\  3014--3023, 2021.

\bibitem[Yang et~al.(2022)Yang, Benaim, Jampani, Genova, Barron, Funkhouser, Hariharan, and Belongie]{yang2022polynomialnerf}
Guandao Yang, Sagie Benaim, Varun Jampani, Kyle Genova, Jonathan Barron, Thomas Funkhouser, Bharath Hariharan, and Serge Belongie.
\newblock Polynomial neural fields for subband decomposition and manipulation.
\newblock \emph{Advances in neural information processing systems (NeurIPS)}, 35:\penalty0 4401--4415, 2022.

\bibitem[Yang et~al.(2021)Yang, Shi, and Ni]{medmnistv1}
Jiancheng Yang, Rui Shi, and Bingbing Ni.
\newblock Medmnist classification decathlon: A lightweight automl benchmark for medical image analysis.
\newblock In \emph{IEEE 18th International Symposium on Biomedical Imaging (ISBI)}, pp.\  191--195, 2021.

\bibitem[Yin et~al.(2020)Yin, Yao, Cao, Li, Zhang, Lin, and Hu]{yin2020disentangled}
Minghao Yin, Zhuliang Yao, Yue Cao, Xiu Li, Zheng Zhang, Stephen Lin, and Han Hu.
\newblock Disentangled non-local neural networks.
\newblock In \emph{European Conference on Computer Vision (ECCV)}, pp.\  191--207. Springer, 2020.

\bibitem[Yu et~al.(2021)Yu, Li, Cai, Sun, and Li]{yu2021s2MLPV2}
Tan Yu, Xu~Li, Yunfeng Cai, Mingming Sun, and Ping Li.
\newblock S2-mlpv2: Improved spatial-shift mlp architecture for vision.
\newblock \emph{arXiv preprint arXiv:2108.01072}, 2021.

\bibitem[Yu et~al.(2022)Yu, Li, Cai, Sun, and Li]{yu2022s2}
Tan Yu, Xu~Li, Yunfeng Cai, Mingming Sun, and Ping Li.
\newblock S2-mlp: Spatial-shift mlp architecture for vision.
\newblock In \emph{Conference on Computer Vision and Pattern Recognition (CVPR)}, pp.\  297--306, 2022.

\bibitem[Yun et~al.(2019)Yun, Han, Oh, Chun, Choe, and Yoo]{yun2019cutmix}
Sangdoo Yun, Dongyoon Han, Seong~Joon Oh, Sanghyuk Chun, Junsuk Choe, and Youngjoon Yoo.
\newblock Cutmix: Regularization strategy to train strong classifiers with localizable features.
\newblock \emph{Conference on Computer Vision and Pattern Recognition (CVPR)}, pp.\  6023--6032, 2019.

\bibitem[Zhang et~al.(2018)Zhang, Cisse, Dauphin, and Lopez-Paz]{zhang2017mixup}
Hongyi Zhang, Moustapha Cisse, Yann~N Dauphin, and David Lopez-Paz.
\newblock mixup: Beyond empirical risk minimization.
\newblock \emph{International Conference on Learning Representations (ICLR)}, 2018.

\end{thebibliography}
\bibliographystyle{iclr2024_conference}
\newpage
\appendix
\appendix
\section*{Contents of the appendix}

The contents of the supplementary material are organized as follows: 
\begin{itemize}
    \item In \cref{sec:poly_mixer_app_proof_proposition}, we provide a technical proofs regarding the interactions learned by \layer{} and \block.
    \item In \cref{appendix:architecture}, we exhibit the details of the architecture design for our model, including the design of single-stage and multi-stage architecture.
    \item In \cref{appendix:PolyMLP}, we present the schematic of the \layer{} to provide a more intuitive understanding of its structure. Additionally, we provide detailed explanations of the \layer{} and \block{} components.
    \item In \cref{appendix:pyramid embedding}, we elaborate on the pyramid patch embedding method mentioned in \cref{subsec:Network Architecture}. We provide illustrative comparisons between the conventional patch embedding approach and our method, showcasing the differences in the representation of input data between the two approaches.
    \item In \cref{ssec:medical image classification}, We extend our method beyond natural images to the medical domain.
    \item In \cref{appendix:Inductive biases study}, we visualize the effective receptive field of our ImageNet-1K pre-trained model. We also visualize related models with same input image as a comparison.
    \item In \cref{appendix:NeuralODE Solver comparsion}, we follow the original evaluation experiment to further compare our method with existing NeuralODE Solver \citep{dupont2019augmented,chen2018neural}.
    \item In \cref{appendix:semantic segmentation}, we scrutinize our model further beyond image recognition. We conduct a semantic segmentation evaluation on ADE20K dataset. The result shows our proposed model also greatly surpasses previous \shortpolynamesingle{} models in semantic segmentation.
    \item In \cref{appendix:remeasure flops}, we measure computation cost using FLOPs. We notice that the tool impacts the number of FLOPs and indeed we observe that this resulted in  less accurate measurements of previous PN models. We utilize open-source code to reevaluate the computational costs of existing \shortpolynamesingle{} models. In this section, we present precise FLOPs (Floating Point Operations) values for a range of \shortpolyname, along with their corresponding performance on ImageNet1K.
     \item In \cref{appedix: Complexity Analysis}, we conduct complexity analysis of our proposed model.
    \item In \cref{appendix:imagenet}, we list the training setting and hyperparameter of our model used for ImageNet-1K training. We also conduct an error analysis using saliency maps in Section \cref{ssec:error analysis} to investigate the weaknesses of our model further.
    \item In \cref{appendix:add med}, we list the training settings and hyperparameters used for the medical image classification experiment and the fine-grained classification experiment. Additionally, we include the model settings table.
    \item In \cref{appendix:initial}, we perform experiments to examine the effects of various parameter initialization methods on the final performance of our model.
\end{itemize}

\newpage

\section{Proofs}
\label{sec:poly_mixer_app_proof_proposition}

In this section, we derive the proof of \cref{proposition:poly_mixer_multiplicative_interaction_per_layer} and also include a new proposition for the \block, along with the associated proof. 
\subsection[Proof of Proposition 1]{Proof of \cref{proposition:poly_mixer_multiplicative_interaction_per_layer}}
\label{ssec:poly_mixer_app_proof_multiplicative_proposition}
In this paper, we use a bold capital letter to represent a matrix, e.g., $\bm{X}, \bm{F}$, while a plain letter to represent a scalar. For instance, $X_{j, \rho}$ represents the element in the $j\myth$ row and the $\rho\myth$ column. The matrix multiplication of two matrices $\bm{F}, \bm{K}$ results in the following $j, \rho$ element: \\$[\bm{F}\cdot \bm{K}]_{j, \rho} = \sum_{q} F_{j,q} K_{q,\rho}$. 

As a reminder, the input is $\bm{X} \in \realnum^{d \times s}$, where $d \in \naturalnum$ is the length of a token and $s\in \naturalnum$ is the number of tokens. Thus, to prove the proposition, we need to show that products of the form $X_{\tau,\rho} X_{\psi,\rho}$ exist between elements of each token. 

\cref{eq:poly_mixer_core_equation_mu_layer} relies on matrix multiplications and a Hadamard (elementwise) product. Concretely, if we express the $j,\rho$ element of $\left(\bm{A} \bm{X}\right) \ast \left(\bm{B} \bm{D}\bm{X} \right)$, we obtain:
\begin{equation}
    [\left(\bm{A} \bm{X}\right) \ast \left(\bm{B} \bm{D}\bm{X} \right)]_{j,\rho} = \sum_{\tau=1,\psi=1}^d \sum_{\omega} A_{j,\tau} X_{\tau,\rho} B_{j,\omega} D_{\omega,\psi} X_{\psi,\rho} \;.
\end{equation}

Then, we can add the additive term and the matrix $\bm{C}$. 
If we express \cref{eq:poly_mixer_core_equation_mu_layer} elementwise, we obtain the following expression:
\begin{equation}
    [\bm{Y}]_{i,\rho} = \sum_{j=1}^m C_{i,j} \left\{\sum_{\tau=1,\psi=1}^d \sum_{\omega} A_{j,\tau} X_{\tau,\rho} B_{j,\omega} D_{\omega,\psi} X_{\psi,\rho} + \sum_{\tau=1}^d A_{j,\tau} X_{\tau,\rho}\right\}\;.
    \label{eq:poly_mixer_core_equation_mu_layer_elementwise_app}
\end{equation}

That last expression indeed contains sums of products $X_{\tau,\rho} X_{\psi,\rho}$ for every output element, which concludes our proof.

\subsection{Interactions of the \block}
\label{ssec:poly_mixer_app_interactions_polyblock_proposition}

As we mention in the main paper, a \block{} comprises of two \layer s, so one reasonable question is how the multiplicative interactions of the \layer{} extend in the \block. We prove below that up to fourth degree interactions are captured in such a block. To simplify the derivation, we focus on two consecutive \layer s as the \block. 

\begin{proposition}
\label{proposition:poly_mixer_fourth_degree_interaction_per_block}
The \block{} captures up to fourth degree interactions between elements of each token.
\end{proposition}

\begin{proof}
    All we need to show is that products of the form $X_{\tau,\rho} X_{\psi,\rho} X_{\gamma,\rho} X_{\epsilon,\rho}$ appear. Using the insights from \cref{ssec:poly_mixer_app_proof_multiplicative_proposition}, we expect the fourth degree interactions to appear in the Hadamard product, so we will focus on the term  $\left(\bm{A}^{(2)} \bm{X}^{(2)}\right) \ast \left(\bm{B}^{(2)} \bm{D}^{(2)}\bm{X}^{(2)} \right)$, where the ${(2)}$ declares the weights and input of the second \layer. 

    The elementwise expression for $\bm{X}^{(2)}$ is directly obtained from \cref{eq:poly_mixer_core_equation_mu_layer_elementwise_app}. To simplify the expression, we ignore the additive term, since exhibiting the fourth degree interactions is enough. 
    The $q, w$ element of the expression is the following:
    \begin{equation}
        \left[\left(\bm{A}^{(2)} \bm{X}^{(2)}\right) \ast \left(\bm{B}^{(2)} \bm{D}^{(2)}\bm{X}^{(2)} \right)\right]_{w,\rho} = \sum_{q,\delta}  \sum_{\xi=1,\theta=1}^m \sum_{\alpha} C_{q, \xi} C_{\delta,\theta} A^{(2)}_{w,q} B^{(2)}_{w,\alpha} D^{(2)}_{\alpha,\delta} \left\{ f_1  \right\} \left\{ f_2  \right\}\;,
    \end{equation}
    where the expression $\left\{ f_1  \right\} \left\{ f_2  \right\}$ is the following:
    \begin{equation}
        \sum_{\tau=1,\psi=1, \gamma=1, \epsilon=1}^d \sum_{\omega, \omega^{\dagger}} A_{\xi,\tau} X_{\tau,\rho} B_{\xi,\omega} D_{\omega,\psi} X_{\psi,\rho}  A_{\theta,\gamma} X_{\gamma,\rho} B_{\theta,\omega^{\dagger}} D_{\omega^{\dagger},\epsilon} X_{\epsilon,\rho}\;.
    \end{equation}
    Notice that the last expression indeed contains products of the form $X_{\tau,\rho} X_{\psi,\rho} X_{\gamma,\rho} X_{\epsilon,\rho}$, which concludes the proof.
\end{proof}

\section{Multi-stage \modelnamePM{}}
\label{appendix:architecture}
\begin{figure}[htbp]
  \centering
  \includegraphics[width=0.9\columnwidth]{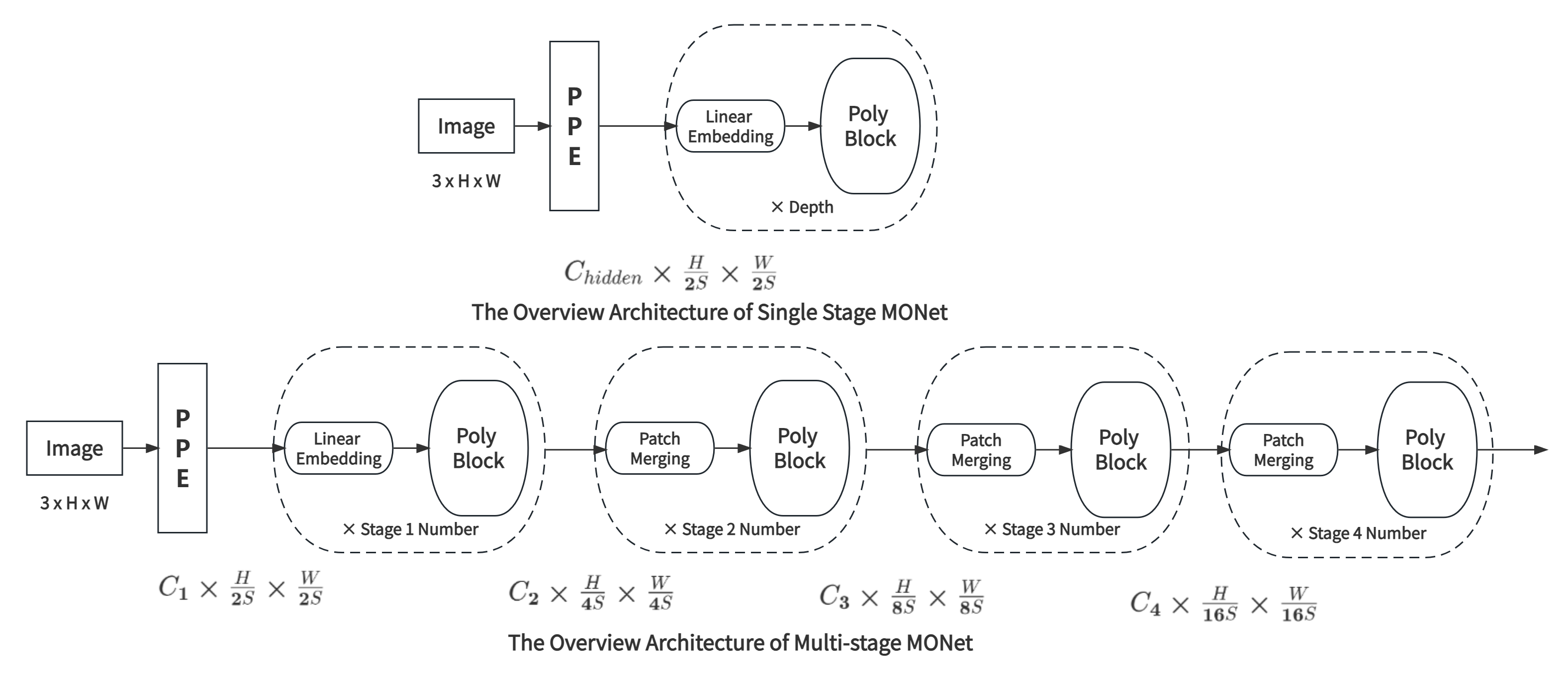}
  \caption{The  Schematic of (simple) \modelnamePM{} and Multi-stage \modelnamePM{}. PPE represents our pyramid patch embedding.} 
  \label{fig:Architecture}
\end{figure}

In deep architectures often a different number of channels or hidden size is used for different blocks~\citep{he2016resnet}. Our preliminary experiments indicate that \modelnamePM{} is amenable to different hidden size as well. Following recent works~\citep{chen2021cyclemlp,guo2022hire}, we set $4$ different hidden sizes across the network (referred to as stages) and we refer to this variant as the `Multi-stage \modelnamePM'. This variant is mostly used for large-scale experiments, since for datasets such as CIFAR10, and CIFAR100 a single hidden size is sufficient.

\section{Details of \layer{} and \block}
\label{appendix:PolyMLP}

\begin{figure}[htbp]
  \centering
  \includegraphics[width=0.7\columnwidth]{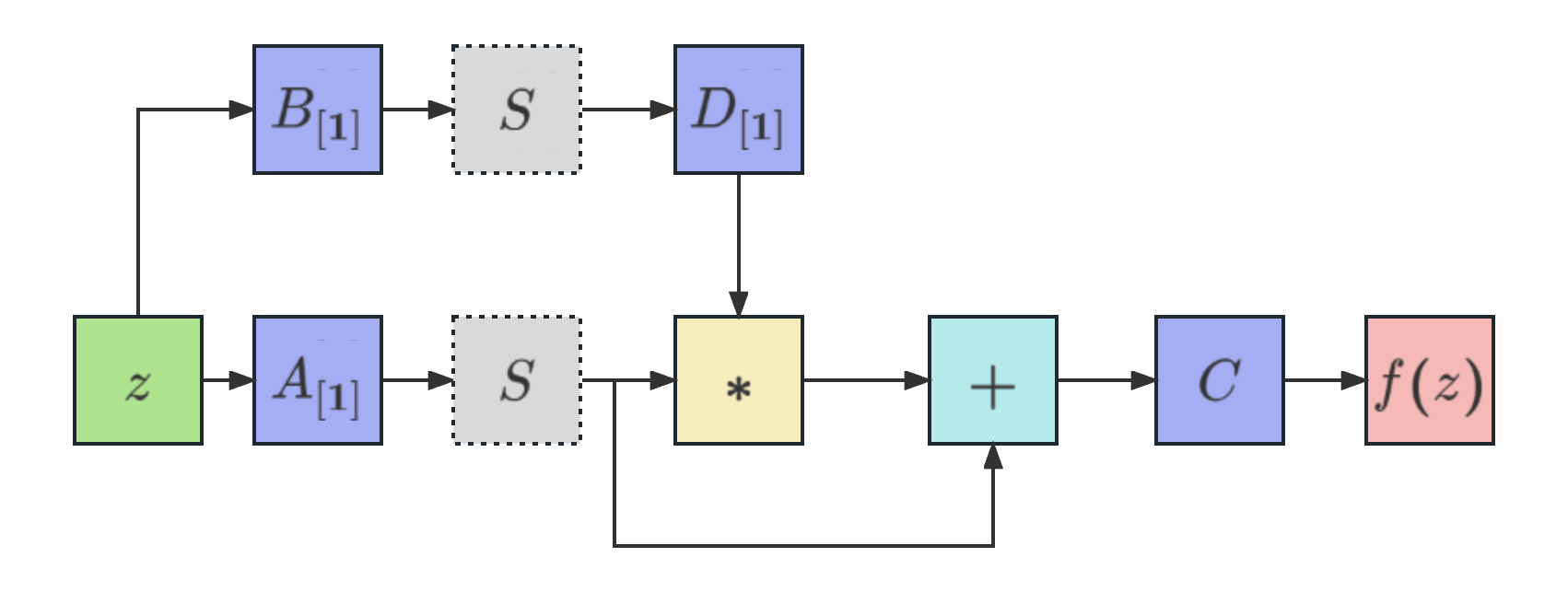}
  \caption{The  Schematic of \layer. Blue boxes correspond to learnable parameters. Green and red boxes denote input and output, respectively. The $\ast$ denotes the Hadamard
product, the $+$ denotes element-wise addition. The gray box denotes the spatial aggregation module, the dotted line represents it as an optional module. In our design the first \layer{} of each \block{} includes a spatial aggregation unit, while the second \layer{} does not.} 
  \label{fig:Poly MLP}
\end{figure}

In \cref{fig:Poly MLP}, we present the schematic of the \layer. The matrices $\bm{A}$, $\bm{B}$, $\bm{C}$, and $\bm{D}$ in \cref{fig:Poly MLP} are described in \cref{subsec:Polymlp}. As mentioned in \cref{subsec:Polymlp}, the \block{} serves as a larger unit module in our proposed model. 
In our implementation, the \block{} consists of two \layer{} modules. Only the first \layer{} in each block incorporates a spatial aggregation module, where we utilize a spatial shift operation \citep{yu2022s2}.

\section{Pyramid Embedding}
\label{appendix:pyramid embedding}

\begin{figure}[htbp]
  \centering
 \includegraphics[width=1\columnwidth]{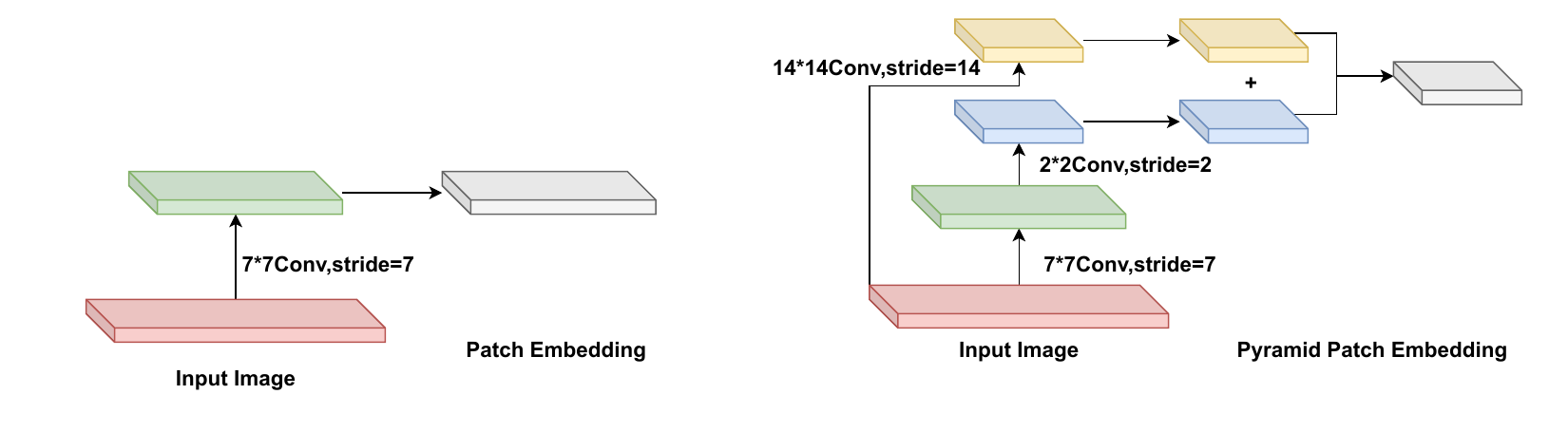}
  \caption{Pyramid Patch Embedding}
  \label{fig:Pyramid Patch Embedding}
\end{figure}
A Patch Embedding, which is typically used in the input space of the model,  
converts an image into a sequence of $N$ non-overlapping patches and projects to dimension $d$. %
If the original image has $(H,W)$ resolution, each patch $X_i$ has resolution $(P,P)$. Then, we obtain $S= H\times W / P^2$ patches, where $X_i \in \realnum^{P\times P\times d}$. For models based on MLP structure, smaller patches can capture finer-grained features better. However, due to the quadratic growth in computation caused by reducing the patch size, most models \citep{guo2022hire,chen2021cyclemlp}  do not follow this design. To reduce the computational burden by smaller patches, these networks assume multiple hidden sizes across the network, see \cref{appendix:architecture} for further details. Instead, we utilize an additional method for reducing the computational cost. After performing the first patch embedding on the original image using a small patch size, we can further reduce the input size by performing $2\times 2$ convolution with stride $2$. With this approach, we can achieve performance comparable to directly using a small patch size while significantly reducing computational costs and allowing us to adopt a more elegant and simple structure.

Inspired by FPN~\citep{lin2017fpn}, we introduce a different level of patch embedding to further enhance the performance. Each level of the pyramid represents an image at a different scale, with the lower levels representing images with higher resolution and the higher levels representing images with lower resolution. We use lower resolution image patch merged with downsampled higher-resolution image patch by element-wise addition. In \cref{subsec:patch embed text}, we compare our method with normal patch embedding. The embedding method we propose not only reduces network computational complexity but also improves the performance when compared to larger patch sizes.

We conduct 8 comparisons to prove the soundness of our method. We employ different patch embedding and patch sizes for the MLP-Mixer and our model. We train those models on ImageNet-100 dataset and report Top-1 Accuracy, their parameters and computation costs (FLOPs) in \cref{tab:polymixer embed} and \cref{tab:mlpmixer embed}.

\begin{table}[tb]
    \label{tab:polymixer_appendix_aux_table_patch_embedding_effect}
 \begin{minipage}{0.45\textwidth}
    \centering
    \setlength{\tabcolsep}{0.4\tabcolsep}
    \caption{Influence of the patch embedding on \modelnamePM.}
    \resizebox{\textwidth}{!}{%
    \begin{tabular}{@{}ccccc@{}}
    \toprule
    \textbf{Method} & \textbf{Patch size } & \textbf{Top-1 Acc(\%)} & \textbf{Paras (M)} & \textbf{FLOPs} \\
    \hline
    Patch Embedding & 7 & 83.08  & 27.77 & 28.14\\
    Patch Embedding & 14 & 79.54  & 27.94 & 7.07\\
    Pyramid Patch Embedding & 7 & 82.04 & 28.02 & 7.23\\
    Pyramid Patch Embedding & 14,7 &82.94 & 28.54 & 7.28 \\
    \hline
    \end{tabular}}
    \label{tab:polymixer embed}
\end{minipage}
\hspace{1cm}
\begin{minipage}{0.45\textwidth}
    \centering
    \setlength{\tabcolsep}{0.4\tabcolsep}
    \caption{The influence of the patch embedding approach on MLP-Mixer.}
    \resizebox{\textwidth}{!}{%
    \begin{tabular}{@{}ccccc@{}}
    \toprule
    \textbf{Method} & \textbf{Patch size } & \textbf{Top-1 Acc(\%)} & \textbf{Paras (M)} & \textbf{FLOPs} \\
    \hline
    Patch Embedding & 7 & 77.92  & 26.60 & 19.44\\
    Patch Embedding & 14 & 77.93  & 23.68 & 4.91\\
    Pyramid Patch Embedding & 7 & 79.64 & 24.50 & 5.18\\
    Pyramid Patch Embedding & 14,7 &79.96 & 24.80 & 5.26\\
    \hline
    \end{tabular}}
    \label{tab:mlpmixer embed}
    \end{minipage}
\end{table}

\cref{tab:polymixer_appendix_aux_table_patch_embedding_effect} indicates that our pyramid patch embedding improves the performance compared to normal patch embedding. By leveraging our proposed method, we can utilize small patch size to achieve better results while preventing quadratic growth of computation costs. This could be used as an ad-hoc module in other MLP models. At the same time, even with a normal patch embedding, our model still outperforms MLP-Mixer.

In the context of merging different levels of features, we initially adopt an elementwise addition approach. Additionally, we explore a U-Net-style approach \citep{ronneberger2015unet}. In this approach, we concatenate the features along the channel dimension and then reduced the dimension using a 1x1 depthwise convolution. However, after conducting experiments, we observe that this approach had limited improvement and was considerably less effective compared to the elementwise addition method.
\section{Medical Image Classification}
\label{ssec:medical image classification}
To assess the efficacy of our model beyond natural images, we conduct an experiment on the MedMNIST challenge across ten datasets~\citep{medmnistv1}. The dataset encompasses diverse medical domains, such as chest X-rays, brain MRI scans, retinal fundus images, and more. MedMNIST serves as a benchmark dataset for evaluating models in the field of medical image analysis, enabling us to evaluate the performance of our model across various medical imaging domains.

\begin{table}[tb]
\centering
\setlength{\tabcolsep}{0.5\tabcolsep}
\caption{\modelnamePM{} performance compared to ResNet18 on the MedMNIST benchmark, with detailed numbers. The best results are marked in \textbf{bold}.}
\resizebox{0.95\textwidth}{!}{%
\begin{tabular}{l|cccc|cccc|cc}
\toprule
Dataset   & \modelnamePMS{} & PDC & $\Pi$-Nets & MLP-Mixer & ResMLP & $S^{2}$MLP-D-S & $S^{2}$MLP-W-S & ResNet & AutoSklearn\\ 
\hline
Path      & \textbf{90.8} & 92.1 & 90.8 & 89.2 & 89.6 & 90.1 & 90.7 & 90.7 & 83.4\\
Derma     & \textbf{77.5} & \textbf{77.5} & 72.9 & 76.1 & 76.5 & 76.2 & 76.9 & 73.5 & 74.9\\
Oct       & \textbf{80.1} & 77.2 & 79.8 & 78.2 & 79.7 & 79.1 & 79.5 & 74.3 & 60.1\\
Pneumonia & \textbf{93.4} & 92.5 & 89.1 & 93.1 & 89.1 & 91.9 & 93.1 & 85.4 & 87.8\\
Retina    & \textbf{55.7} & 53.6 & 52.2 & 54.5 & 54.7 & 53.7 & 54.0 & 52.4 & 51.5\\
Blood     & \textbf{96.7} & 95.5 & 94.5 & 94.7 & 95.3 & 95.8 & 95.2 & 93.0 & 96.1\\
Tissue    & \textbf{67.7} & 67.5 & - & 65.9 & 66.8 & 65.9 & 66.0 & 67.7 & 53.2\\
OrganA    & \textbf{93.6} & 93.5 & 92.5 & 90.5 & 91.4 & 92.3 & 92.7 & 93.5 & 76.2\\
OrganC    & 88.9 & 93.0 & 89.3 & \textbf{90.6} & 87.5 & 88.9 & 89.4 & 90.0 & 87.9\\
OrganS    & \textbf{78.5} & 77.9 & 75.0 & 76.9 & 75.7 & 77.5 & 78.4 & 78.2 & 67.2\\
\bottomrule
\end{tabular}}
\label{tab:medmnist}
\end{table}

In our experiments, we train the variant of \modelnamePMS{} with 14.0M parameters and 0.68G FLOPs.
The results are shown in \cref{tab:medmnist} exhibit that our model outperforms other models in eight datasets.

\section{Visualization of representations from learned models}
\label{appendix:Inductive biases study}
In this section, we use visualizations to understand the difference between the learned models of MLP-Mixer, ResMLP and CNN architectures. The effective receptive field refers to the region in the input data that a neural network's output is influenced by, which has been proved as an effective approach to understand where a learned model focuses on. Specifically, we visualize the hidden unit effective receptive field of our model trained on ImageNet1k, i.e., the output of the last layer, and compare it with pretrained MLP-Mixer, ResMLP and Resnet in \cref{fig:erf}. We adopt a random image from ImageNet1K as input and visualize the effective receptive field. We can notably observe that due to the flattening of input tokens into one-dimensional structures by MLP-Mixer and ResMLP, the features they learn exhibit a distinct grid effect due to the loss of 2D information. They also tend to emphasize low-level texture information to a greater extent. ResNet's effective receptive field  is more discrete, encompassing both background and foreground elements to some extent, focusing on the global context. On the contrary, \modelnamePM{} focuses on the semantic parts (e.g. on the dog’s face) which achieves a balance between global and local context.
\begin{figure}[htbp]
\captionsetup{font=small}
  \centering
 \includegraphics[width=0.7\columnwidth]{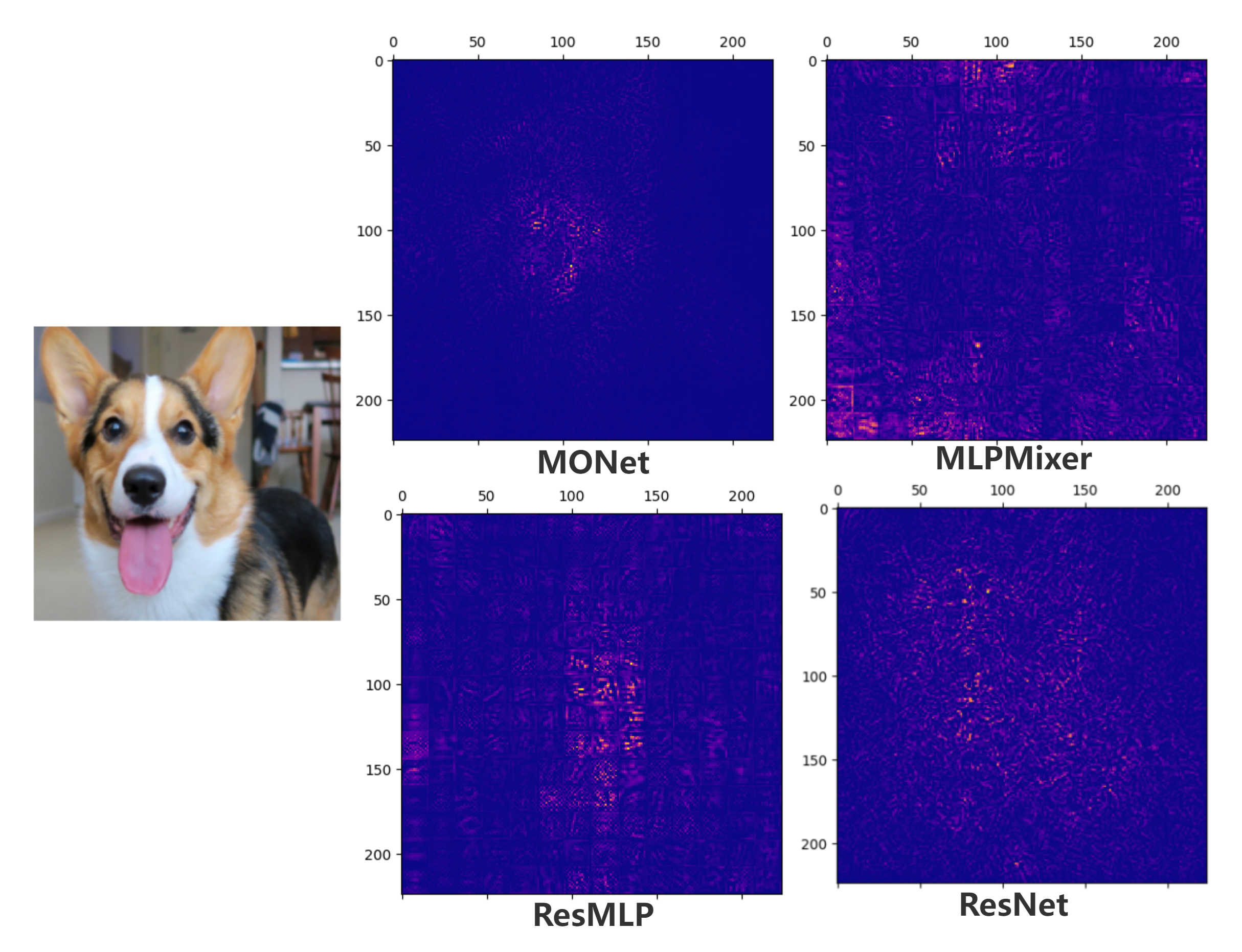}
   \caption{Input image(\textbf{Left}). The ERF of model pretrained on ImageNet-1K(\textbf{Right}). We visualize the effective receptive field of MLP-Mixer, ResMLP, ResNet and \modelnamePM{}. Our model more focus on the dog's face and the outline shape of dog, while the rest MLP models more focus on the low-level texture feature.}
\label{fig:erf}
 \end{figure}

\section{NeuralODE Solver comparsion}
\label{appendix:NeuralODE Solver comparsion}
Following the evaluation setup of the Neural-ODE network, we implement a Poly-Neural ODE based on our method, which can be used to solve ordinary differential equations for simulating physics problems. We conduct scientific computing experiments using simulated data, with the aim of predicting and interpreting the dynamics of physical systems

The specific task is to simulate moving the \textbf{randomly generated} inner sphere particles out of the annulus by pushing against and stretching the barrier. Eventually, since we are mapping discrete points and not a continuous region, the flow is able to break apart the annulus to let the flow through. The initial particle states are shown in Figure \cref{fig:initial}. 
\begin{figure}[htbp]
 \centering
    \includegraphics[width=0.3\textwidth]{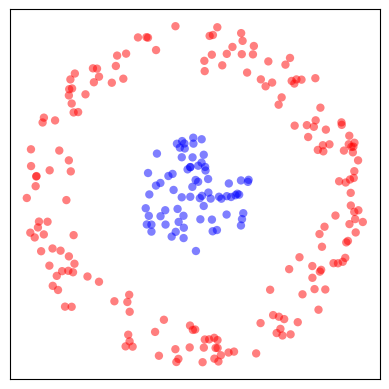}
    \caption{The initial particles states. The task is to move the inner sphere of blue particles out of the red particles.}
    \label{fig:initial}
\end{figure}

We set the model with hidden dimension 32 and train on simulated data for 12 epoches. The revolution of the first 10 epochs are shown in \cref{fig:move}.
\begin{figure}[htbp]
 \centering
    \includegraphics[width=0.7\textwidth]{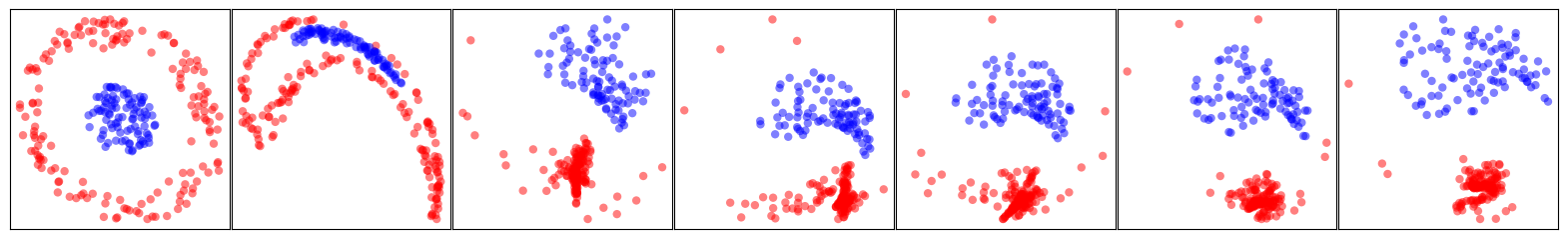}
    \caption{The move of particles. From left to right it's feature evolution with iterations increase. }
    \label{fig:move}
\end{figure}
\begin{figure}[thbp]
 \centering
    \includegraphics[width=0.3\textwidth]{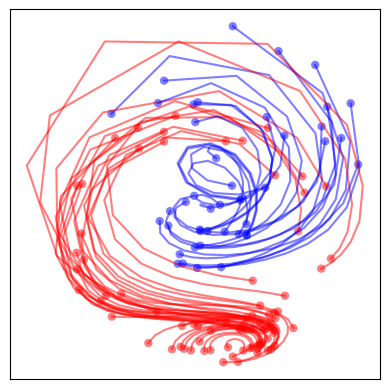}
    \caption{The predicted trajectory of particles. The lines are their history trajectory.}
    \label{fig:traj}
\end{figure}

We implement two models based on NeuralODE \citep{chen2018neural} and  Augmented Nerual ODE \citep{dupont2019augmented}, called Poly NerualODE(PolyNODE) and Poly Augmented NeuralODE(Poly ANODE). For Poly NeuralODE, higher-order model could achieve better accuracy with the cost of higher computation cost, the following experiment are based on minimal order $2$ Poly NeuralODE.

We first compare NODE with Poly NODE, the loss plots shown in \cref{fig:ode loss}. We can observe that with 40 epochs training, The Poly NODE approximates the functions faster and more accurately.  
\begin{figure}[htbp]
 \centering
    \includegraphics[width=0.6\textwidth]{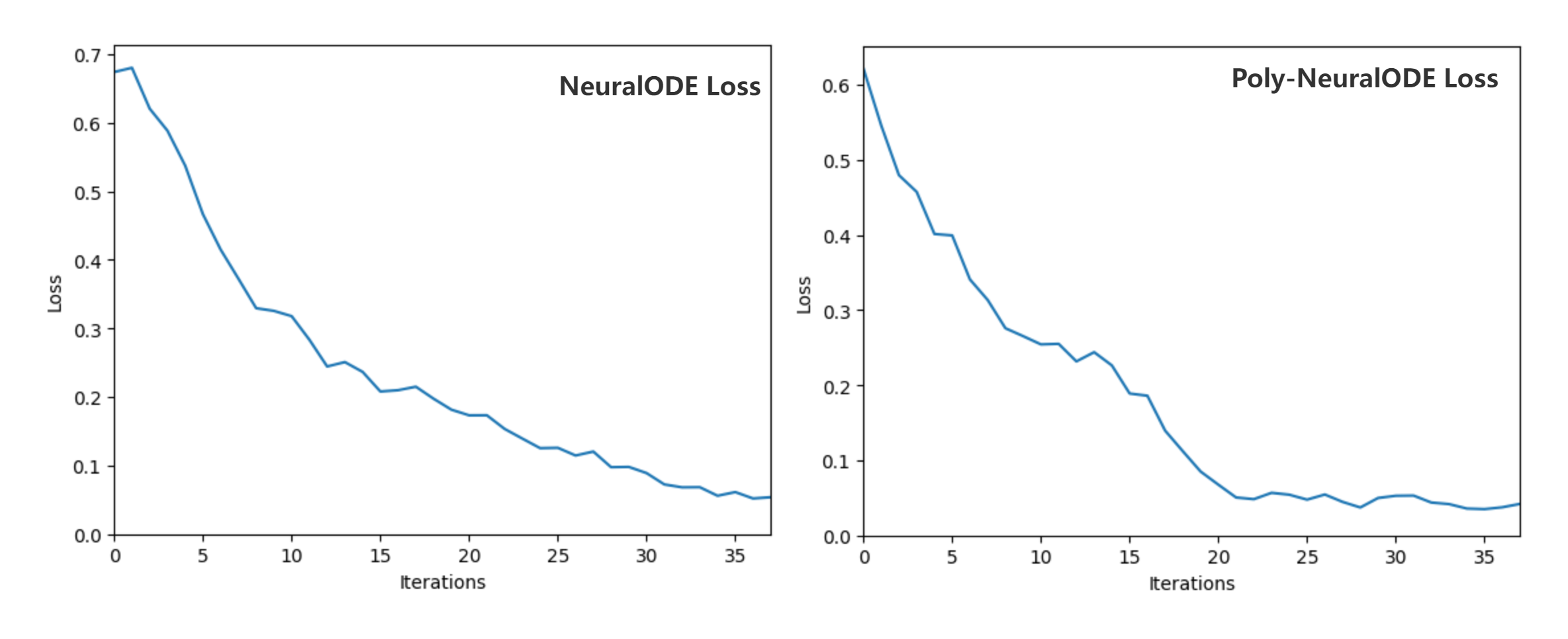}
    \caption{ Loss plots for NODEs and ANODEs trained on  2d data. Poly NODEs easily approximate the functions and are consistently faster than NODEs. }
    \label{fig:ode loss}
\end{figure}
\begin{figure}[htbp]
 \centering
    \includegraphics[width=0.6\textwidth]{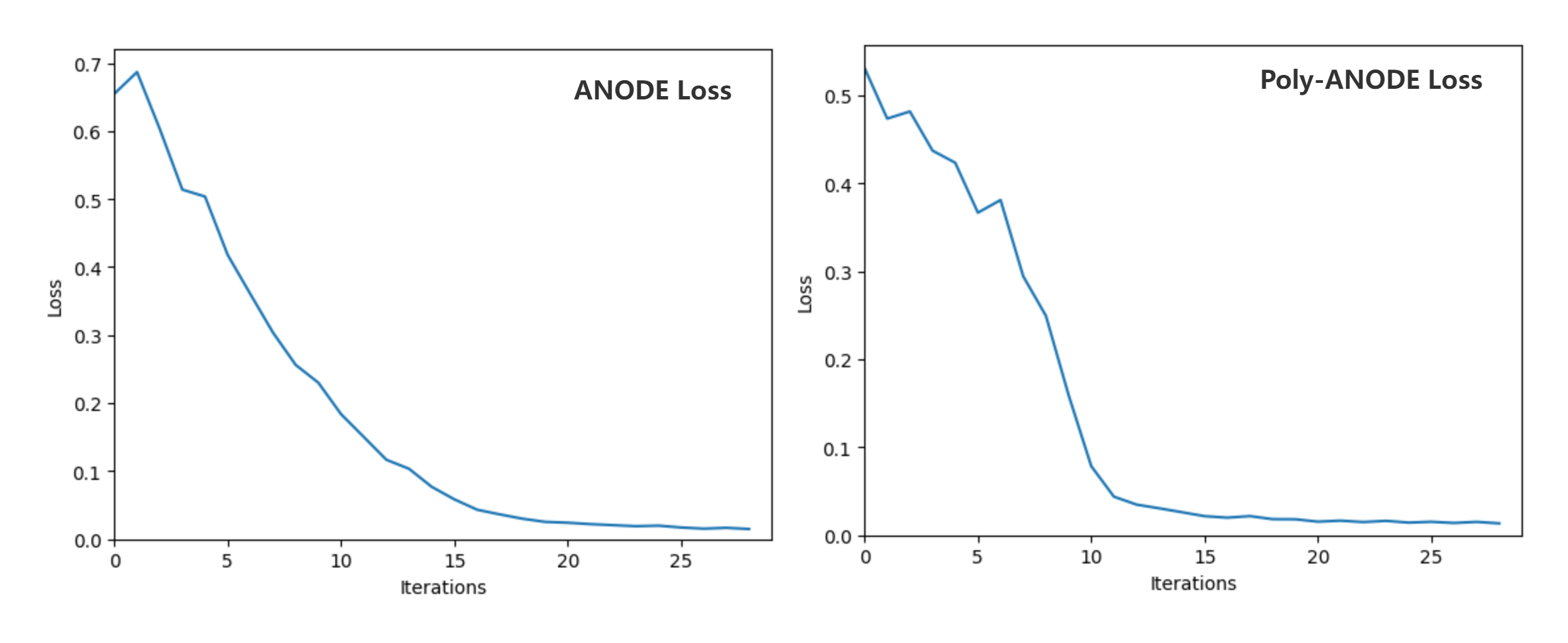}
    \caption{ Loss plots for Poly ANODEs and ANODEs trained on  2d data. Poly NODEs easily approximate the functions and are consistently faster than NODEs. }
    \label{fig:anode loss}
\end{figure}
We then compare ANODE with Poly ANODE. The loss plot trained 30 epoches is shown in \cref{fig:anode loss}. ANODEs augment the space on which the ODE is solved, allowing the model to use the additional dimensions to learn more complex functions using simpler flows. Our PolyANODE inherent its advantage while converge faster.

\section{Semantic Segmentation}
\label{appendix:semantic segmentation}
\textbf{Settings.} To further explore the performance of \modelnamePM{} in downstream task. We conduct semantic segmentation experiments on ADE20K dataset and present the final performance compared to previous models. The \cref{tab:semantic} shows the result. Previous PNN models exhibited poor performance on downstream tasks. Our model has overcome this issue and achieved results comparable to some state-of-the-art models.

\begin{table}[htbp]
\label{tab:semantic}
\centering
\caption{\textbf{Semantic segmentation on ADE20K val},All models use Semantic FPN head. The transformer and ResNet models result from the original paper. The polynomial networks result from the reproduced experiment.}
\begin{tabular}{c|c}
\hline
\multirow{2}{*}{Backbone} & Semantic FPN \\ \cline{2-2} 
                          & mIoU(\%)     \\ \hline
ResNet18 \citep{he2016resnet}                  & 32.9         \\ 
FCN \citep{long2015fully}& 29.3\\ \hline
PvT-Tiny \citep{wang2021PVT}                 & 35.7         \\
Seg-Former \citep{xie2021segformer}                & 37.4         \\ \hline
R-PDC \citep{chrysos2022augmenting}                    & 20.7         \\ 
\newmodelnamePI \citep{chrysos2023regularization}                & 19.4         \\ 
Multi-stage \modelnamePMB{}        & 37.5        \\ \hline
\end{tabular}

\end{table}

\section{Computation Cost compared to previous polynomial models}
\label{appendix:remeasure flops}
FLOPs (Floating-Point Operations Per Second) is commonly used metric to measure the computational cost of a model. However, tools like flops-counters and mmcv, which use PyTorch's module forward hooks, have significant limitations. They are accurate only if every custom module includes a corresponding flop counter. 

In this paper, we adopt fvcore Flop Counter\footnote{\href{https://github.com/facebookresearch/fvcore/blob/main/docs/flop_count.md}{fvcore flops counter}} developed by Meta FAIR, which provides the first operator-level FLOP counters. This tool observes all operator calls and collects operator-level FLOP counts. For models like polynomial models that involve numerous custom operations, an operator-level counter will give more accurate results. We reproduced previous polynomial models according to their open-source code and re-measure their computation cost with fvcore flops counter. 
This causes a slight difference from the FLOPs reported in previous \shortpolyname. 
For instance, operations such as custom normalization modules and Hadamard product computations were often overlooked due to former tool limitations. \emph{We adopt the  number taken from original papers in \cref{tab:imagenet-benchmark-updated} and we have corrected and updated the FLOPs report for recent polynomial models in the table below.} In comparison to previous work, our model has significant performance benefits while \emph{only incurring 20\% of the computational cost} of the previous state-of-the-art work.
Simultaneously, our performance on ImageNet-1K surpasses the previous state-of-the-art by \textbf{10\%}.
\begin{table}[htbp]
\centering
\caption{The re-measure computational costs for polynomial models, along with their corresponding Top-1 accuracy on ImageNet-1K}
\begin{tabular}{c|c|c}
\hline
Model                & FLOPs(B) & Accuracy(\%) \\ \hline
PDC                  & 30.78    & 71.0         \\ \hline
R-PDC                & 39.56    & -            \\ \hline
\newmodelnamePI           & 29.34    & 70.2         \\ \hline
\modelnamedense           & 26.41    & 70.0         \\ \hline
\modelnamePMS{} (Ours)         & 3.6      & 76.0         \\ \hline
\modelnamePMB{}(Ours)          & 11.2     & 79.6         \\ \hline
Multi-stage \modelnamePMS{}(Ours)  & 2.8      & 77.0         \\ \hline
Multi-stage \modelnamePMB{}(Ours)  & 6.8      & 81.3         \\ \hline
\end{tabular}
\end{table}

\section{Complexity Analysis}
\label{appedix: Complexity Analysis}
To simplify the process, we conduct a complexity analysis of  \modelnamePM{} adopted the single-stage design as follows. 

\textbf{Pyramid Patch Embedding Layer (PPEL)} crops the input raw image $I \in \mathbb{R}^{H \times W \times 3}$ where $W$ is the width and $H$ is the height into several non-overlapping patches of dimension $P \in \mathbb{R}^{p \times p \times 3}$. Consider 1-level embedding,Our method conducts compression to the patches with a $2 \times 2$ convolution layer, resulting in the  token $W_{0} \in \mathbb{R}^{p\times p \times c}$, token $W_1\in \mathbb{R}^{\frac{p}{2}\times\frac{p}{2}\times c}$,where $W_0$ is the output of the first convolution layer, $W_1$ is the output of the second convolution layer. Noted that the  final number of patches are $\bm{N_p} = \frac{W}{2p} \times \frac{H}{2p}$ which is 25\% of normal patch embedding. Thus, the total number of parameters of the Pyramid embedding is:
\begin{equation}
    Params_{PPEL} = ({3}p^2+1)c + c(4c+1) = ({3}p^2+4c + 2)c
\end{equation}

 In summary the floating operations in PPEL is 
\begin{equation}
    FLOPs_{PPEL} = 3\times\frac{W}{p}\times\frac{H}{p} cp^2 + 4\times c^2 \times \frac{W}{2p}\times\frac{H}{2p} = 4c{N_p}(3p^2+c)
\end{equation}

\textbf{Poly Blocks(\textbf{PB})}: Our model consists of $N$ identical Poly Blocks. Each blocks contains 2 Poly Layer. The first Poly Layers\textbf{($PL_{1}$)} contains two Spatial Shift module in two branches, and the second Poly Layers doesn't contain Spatial Shift module. We denotes that first Poly Layer four fully connected layer as ${W_1,b_1},{W_2,b_2},{W_3,b_3},{W_4,b_4}$ and shrinkage ratio as $s$, where $W_1\in \mathbb{R}^{c\times c},W_2\in \mathbb{R}^{c\times\frac{c}{s}},W_3\in \mathbb{R}^{\frac{c}{s}\times c},W_4\in \mathbb{R}^{c\times c}$. Therefore, the total parameter numbers of  the first Poly Layers is 
\begin{equation}
    Params_{PL_1} = c\times(2c+ 2\frac{c}{s}) +3c +  \frac{c}{s} = c^2(2+\frac{2}{s}) + c(3+\frac{1}{s})
\end{equation}
The second Poly Layer is similar, while we have a expansion rate $r$ similar to other MLP models.We denotes that second Poly Layer four fully connected layer as ${W_5,b_5},{W_6,b_6},{W_7,b_7},{W_8,b_8}$ and shrinkage ratio as $s$, where $W_5\in \mathbb{R}^{c\times rc},W_2\in \mathbb{R}^{c\times\frac{rc}{s}},W_3\in \mathbb{R}^{\frac{rc}{s}\times rc},W_4\in \mathbb{R}^{rc\times c}$. We get the total parameter numbers of the second Poly Layers is 
\begin{equation}
    Params_{PL_2} = c\times(2rc+ \frac{rc}{s}+\frac{r^2c}{s}) +2rc +c + \frac{rc}{s}= c^2(2r+\frac{r+r^2}{s}) + c(2r+\frac{1}{s}+1)
\end{equation}
Hence, the total number of a Poly Block is
\begin{equation}
    Params_{PB} = Params_{PL_1} + Params_{PL_2} = c^2(2r+2+\frac{r^2+r+2}{s})+c(4+2r+\frac{2}{s}) 
\end{equation}
Besides the fully connected layer, the hadamard product bring extra computation where the flops of hadamard product of two matrices in shape $n\times m$ is $n\times m$. The overall hadamard product flops is 
\begin{equation}
    FLOPs_{Prod} = {N_p}c + {N_p}rc = {N_p}c(1+r)
\end{equation}
and the flops becomes
\begin{equation}
    FLOPs_{PB} = {N_p}c^2(2r+2+\frac{r^2+r+2}{s}) + FLOPs_{Prod} = {N_p}c^2(2r+2+\frac{r^2+r+2}{s}) +{N_p}c(1+r)
\end{equation}

\textbf{Fully-connected classification layer(FCL)}: takes input $c$-dimensional vector and feedforward to average pooling layer. The output vectore is in $k$ dimension where $k$ is the number of classes. In summary, the number of parameters of FCL is:
\begin{equation}
    Params_{FCL} = k(c+1)
\end{equation}
The flops of FCL is:
\begin{equation}
    FLOPs_{FCL} = {N_p}ck
\end{equation}

\textbf{Overall architecture}: We conclude that the summary of parameters of overall of architecture. The total numbers of parameters of our architecture is:
\begin{equation}
    Params = Params_{PPEL} + N\times Params_{PB} + Params_{FCL}
\end{equation}
The flops of our architecture is:
\begin{equation}
    FLOPs = FLOPs_{PPEL} + N\times FLOPs_{PB} + FLOPs_{FCL}
\end{equation}
\newpage
\section{\imagenet{} Classification}
\label{appendix:imagenet}
\subsection{\imagenet{} Training setting}
 The \cref{tab:imagenettrain} shows our experiment setting for training \modelnamePM{} on ImageNet1K dataset. Our implementation for the optimizer and data augmentation follow the popular timm library\footnote{\href{https://github.com/huggingface/pytorch-image-models}{Hugging Face timm}}.
\label{ssec:imagenettrain}
\begin{table}[htbp]
\label{tab:imagenettrain}
\caption{\imagenet{} Training Settings in \cref{ssec:poly_mixer_experiment_imagenet}}
\centering
\begin{tabularx}{0.5\textwidth}{c|>{\centering\arraybackslash}X}
\toprule
 & \imagenet{} Training Setting \\
\midrule
optimizer & AdamW \\
base learning rate & 1e-3 \\
\multirow{2}{*}{weight decay} &
0.01 (Multi-stage \modelnamePM{}-T) \\&0.02 (Multi-stage \modelnamePM{}-S) \\
batch size & 448 \\
training epochs & 300 \\
learning rate schedule & cosine \\
warmup & \ding{51} \\
label smoothing & 0.1 \\
auto augmentation & \ding{51} \\
random erase & 0.1 \\
cutmix & 0.5 \\
mixup & 0.5 \\
\bottomrule
\end{tabularx}
\end{table}
\subsection{Error Analysis}
\label{ssec:error analysis}
We utilize our best-performing \modelnamePM{} model to compute per-class accuracy rates for all 1000 classes on the validation dataset of \imagenet. In \cref{tab:misImagenet}, we present the top 10 least accurate and misclassified classes. Additionally, in \cref{fig:mis1} we showcase the images of the most misclassified class (laptop). In \cref{fig:mis2}, we show a failure case of the screen. In \cref{fig:mis4}, we show a failure case of tennis. In \cref{fig:corr5}, we show a successful case of small object sunglasses.

\begin{table}[htbp]
\caption{The top-10 least accurate classes and their labels}
\label{tab:misImagenet}
\centering
\begin{tabular}{|c|c|c|c|}
\hline %
Class Name & Accuracy (\%) & Class Name & Accuracy (\%) \\ \hline
tiger Cat        & 20     &  laptop,laptop computer          & 24\\\hline
screen        & 24     &  chiffionier,commode        & 28\\\hline
sunglasses        & 28     &    cassette player       & 30\\\hline
letter opener        & 30     &  malliot        & 32\\\hline
projectile,misslle        & 32     &  spotlight         & 32\\ \hline %
\end{tabular}
\end{table}
\begin{figure}[htbp]
 \centering
    \includegraphics[width=0.7\textwidth]{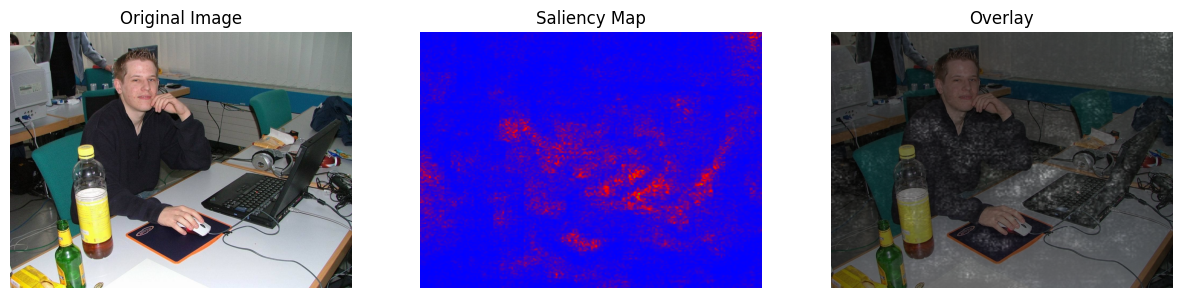}
    \caption{Failure case of \modelnamePM{} as explained from the saliency map. The red figure in the middle denotes the saliency map, which is the area where the model focuses on. The correct class for the image is \textcolor{red}{laptop}, while the saliency map shows the model focuses on the person. This is not an unreasonable error though, since the main figure lies in the center of the image.} 
    \label{fig:mis1}
\end{figure}

\begin{figure}[htbp]
 \centering
    \includegraphics[width=0.7\textwidth]{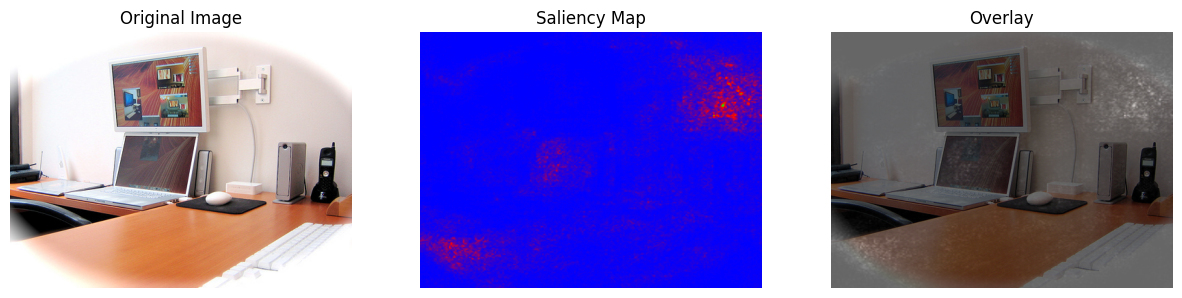}
    \caption{Failure case of \modelnamePM{} as explained from the saliency map. The red figure in the middle denotes the saliency map, which is the area where the model focuses on. The correct class for the image is \textcolor{red}{screen}, while the saliency map shows the model focusing on the wall. }
    \label{fig:mis2}
\end{figure}
\begin{figure}[htbp]
 \centering
    \includegraphics[width=0.7\textwidth]{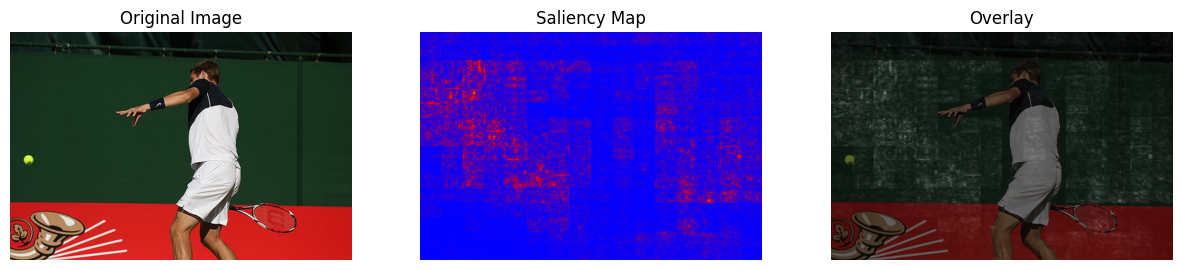}
    \caption{Failure case of \modelnamePM{} as explained from the saliency map. The red figure in the middle denotes the saliency map, which is the area where the model focuses on. The correct class for the image is \textcolor{red}{tennis}, while the saliency map shows the model focusing on the background wall, instead of on the small object. }
    \label{fig:mis4}
\end{figure}
\begin{figure}[htbp]
 \centering
    \includegraphics[width=0.7\textwidth]{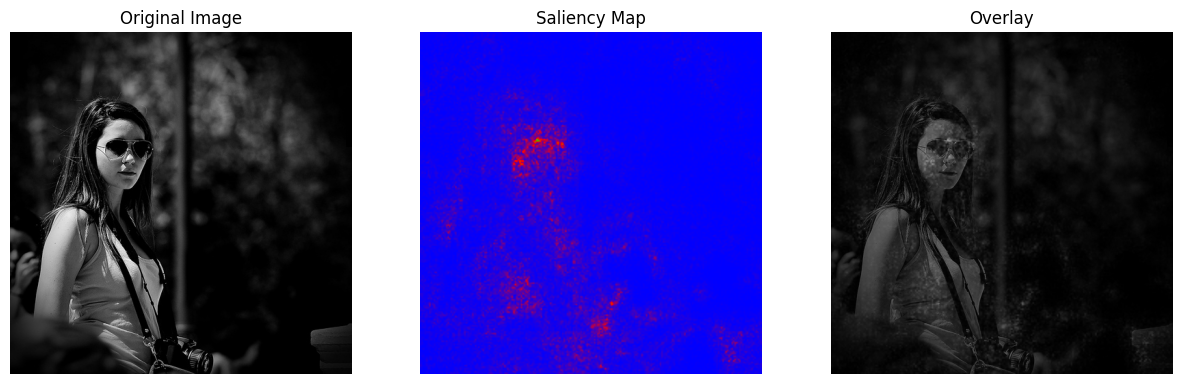}
    \caption{Success case of \modelnamePM{} as explained from the saliency map. The correct class for the image is \textcolor{red}{sunglasses}, The saliency map shows the model focusing on the face of the woman. In some cases, successful classification of a category can be achieved based on global information, even if the object occupies only a small portion of the image. }
    \label{fig:corr5}
\end{figure}
Our analysis dictates that the majority of classification failures occur when the model fails to focus on the main object. This is even more pronounced in cases where the classification label is provided by an item occupying a small portion of the image. We believe this is an interesting phenomenon that is shared with similar models relying on linear operations and not using convolutions. 
\newpage

\subsection{\rebuttal{Sources for baselines}}

\rebuttal{To further increase the transparency of our comparisons, we clarify the details on the main results on \imagenet. We aim for reporting the accuracy and the other metrics using the respective papers or public implementations. Concretely: }
\begin{itemize}
    \item \rebuttal{The paper of \citet{chrysos2022augmenting} is used as the source for PDC.}
    \item \rebuttal{\citet{chrysos2023regularization} is used as the source for ResNet18,Pi-Nets, Hybrid Pi-Nets, \newmodelnamePI{} and \modelnamedense. We clarify here once more that \modelnamePI{} refers to the model \emph{without} activation functions, while `Hybrid \modelnamePI' refers to the model with activation functions.}
    \item \rebuttal{\citet{yu2022s2} is used as the source of ResNet50 performance. We updated the number to 77.2\% (on \imagenet) to align with \cite{he2016resnet}.}
    \item  \rebuttal{The performance data for the other models in the \cref{tab:imagenet-benchmark-updated} are sourced from their respective papers.}
\end{itemize}

\section{Additional benchmarks and Medical Image Classification}
\label{appendix:add med}
We introduce dataset details including dataset information for various datasets used for the empirical validation of our method.

In this section, we give a brief introduction to the dataset we use and the overview statistics of datasets are shown in \cref{tab:dataset overview}.

\textbf{CIFAR10}: CIFAR-10 \citep{krizhevsky2009cifar10} is a well-known dataset in the field of computer vision. It consists of 60,000 labeled images that are divided into ten different classes. Each image in the CIFAR-10 dataset has a resolution of 32x32 pixels and is categorized into one of the following classes: airplane, automobile, bird, cat, deer, dog, frog, horse, ship, or truck. The images in the CIFAR10 dataset are of size $32\times 32$ pixels. We keep its original resolution for training.

\textbf{SVHN}:SVHN \citep{netzer2011svhn} stands for Street View House Numbers. It is a widely used dataset for object recognition tasks in computer vision. The SVHN dataset consists of images of house numbers captured by Google's Street View vehicles. These images contain digit sequences that represent the house numbers displayed on buildings. It contains over 600,000 labeled images, which are split into three subsets: a training set with approximately 73,257 images, a validation set with around 26,032 images, and a test set containing 
approximately 130,884 images. The images in the SVHN dataset are of size $32\times 32$ pixels, we keep its original resolution for training.

\textbf{Oxford Flower}: The Oxford Flower \citep{oxfordflower} dataset, also known as the Oxford 102 Flower dataset, is a widely used collection of images for fine-grained image classification tasks in computer vision. It consists of 102 different categories of flowers, with each category containing a varying number of images. The Oxford Flower dataset provides a challenging testbed for researchers and practitioners due to the high intra-class variability among different flower species. This variability arises from variations in petal colors, shapes, and overall appearances across different species. The images in the Oxford Flower dataset are of size $256\times 256$ pixels, in our training, we use bicubic interpolation to resize all images to $224\times 224$ pixels.

\textbf{Tiny-Imagenet}: Tiny ImageNet\citep{le2015tinyimagenet} is a dataset derived from the larger ImageNet dataset, which is a popular benchmark for object recognition tasks in computer vision. The Tiny ImageNet dataset is a downsized version of ImageNet, specifically designed for research purposes and computational constraints. While the original ImageNet dataset consists of millions of images spanning thousands of categories, the Tiny ImageNet dataset contains 200 different classes, each having 500 training images and 50 validation and test images. This results in a total of 100,000 labeled images in the dataset. The images in the Tiny ImageNet dataset are of size $64\times 64$ pixels, we keep its original resolution for training.

\textbf{MedMNIST}: MedMNIST \citep{medmnistv1} is a specialized dataset designed for medical image analysis and machine learning tasks. It is inspired by the popular MNIST dataset, but instead of handwritten digits, MedMNIST focuses on medical imaging data. The MedMNIST dataset includes several sub-datasets, each corresponding to a different medical imaging modality or task. Some examples of these sub-datasets are ChestX-ray, Dermatology, OCT (Optical Coherence Tomography), and Retinal Fundus. Each sub-dataset contains labeled images that are typically 28x28 pixels in size, resembling the format of the original MNIST dataset. We keep the original size for training. Due to the presence of both the RGB sub-dataset and grayscale sub-dataset in MedMNIST, we employed different model configurations to accommodate this variation. The PathMNIST, DermaMNIST, BloodMNIST, and RetinaMNIST are RGB images, the rest 7 dataset are greyscale images. 

\textbf{ImageNet-C}: ImageNet-C \citep{hendrycks2018benchmarking} is a dataset of 75 common visual corruptions. This dataset serves as a benchmark for evaluating the resilience of machine learning models to different forms of visual noise and distortions. This benchmark provides a more comprehensive understanding of a model's performance and can lead to the development of more robust algorithms that perform well under a wider range of scenarios. The metric used in ImageNet-C is mean corruption error, which is the average  error rate of 75 common visual corruptions of ImangeNet validation set. The lower the mCE, the better the result.

To ensure a fair comparison with the original paper that presented ResNet18, we followed the training protocol outlined in the MedMNIST paper. We trained our model for 100 epochs without any data augmentation, using early stopping as the stopping criterion.

\begin{table}[htbp]
\label{tab:dataset overview}
\caption{The overview of the dataset we use in \cref{ssec:medical image classification} and \cref{ssec:poly_mixer_experiment_fine_grained}. The * in the below table since MedMNIST dataset consists of multiple sub-datasets, each containing medical images from specific categories. These sub-datasets have varying numbers of training samples, testing samples, and classes.}
\setlength{\tabcolsep}{0.3\tabcolsep}
\resizebox{1\columnwidth}{!}{%
\begin{tabular}{|c|c|c|c|c|c|}
\toprule
Dataset       & CIFAR10        & SVHN    & Oxford Flower & Tiny-Imagenet  & MedMNIST       \\ \hline
Classes       & 10             & 10      & 102           & 200            & *              \\ \hline
Train samples & 50000          & 73257   & 1020          & 100000         & *              \\\hline
Test samples  & 10000          & 26032   & 6149          & 10000          & *              \\\hline
Resolution    & 32             & 32      & 256           & 64             & 28             \\\hline
Attribute     & Natural Images & Numbers & Flowers       & Natural Images & Medical Images \\\hline
Image Type    & RGB            & RGB     & RGB           & RGB            & RGB+Greyscale \\ \bottomrule
\end{tabular}}
\end{table}

\section{Initialization}
\label{appendix:initial}
We evaluate various initialization methods for the parameters. The results in \cref{tab:initialization} indicate that the Xavier Normal initialization method yields the most favorable results. By adopting the Xavier Normal initialization, we observe improvements in the performance of the model trained on CIFAR-10.  Our result in the main paper use Xavier Normal as our default initialzation.
Previous works on PNs~\citep{chrysos2022augmenting, chrysos2023regularization} have demonstrated the crucial role of appropriate parameter initialization in the final performance. In this section, we explore different parameter initialization methods for the linear layers in the model and train them on the CIFAR-10 dataset from scratch. To minimize the impact of confounding factors, all experiments are conducted without any data augmentation techniques or regularization methods. The results are shown in \cref{tab:initialization}.
\begin{table}[htbp]
    \centering
     \caption{Comparison of Initialization Methods. The best performance is marked in \textbf{bold}. Note that * in the table uses PyTorch default initialization, it depends on the layer type. For linear layer of shape $(out,in)$,  the values are initialized from $\mathcal{U}(-\sqrt{k},\sqrt{k})$, where $k=\frac{1}{in}$.}
     \label{tab:initialization}
        \begin{tabular}{|c|c|c|}
            \hline
            Initialization Method & Top-1 Acc (\%) & Top-5 Acc (\%) \\
            \hline
            Xavier Uniform \citep{glorot2010xavier}& 88.85 & 99.38 \\
            Xavier Normal \citep{glorot2010xavier}& \textbf{89.13} & \textbf{99.85} \\
            Kaiming Uniform \citep{he2015delving}& 88.56 & 98.70 \\
            Kaiming Normal \citep{he2015delving}& 88.72 & 99.46 \\
            Lecun Normal \citep{lecun2002efficient}& 88.95 & 99.51 \\
            Normal  & 88.06 & 99.34\\
            Sparse \citep{martens2010sparse} & 88.18 & 99.48\\
            Pytorch default*& 88.37& 99.49\\
            \hline
        \end{tabular}
\end{table}

\end{document}